\documentclass{article}

\usepackage{times}
\usepackage{graphicx} 
\usepackage{subfigure}
\usepackage{natbib}
\usepackage{algorithm}
\usepackage{algorithmic}
\usepackage{amsmath}
\usepackage{amsthm}
\usepackage{amssymb}
\usepackage{hyperref}
\usepackage{enumitem}

\renewcommand{\algorithmicrequire}{\textbf{Input:}}
\renewcommand{\algorithmicensure}{\textbf{Output:}}

\theoremstyle{definition}
\newtheorem{definition}{Definition}

\newtheorem{theorem}{Theorem}
\newtheorem{lemma}[theorem]{Lemma}
\newtheorem{proposition}[theorem]{Proposition}
\newtheorem{corollary}[theorem]{Corollary}
\newtheorem{remark}{Remark}

\setlist{nolistsep}

\usepackage[accepted]{icml2014} 

\icmltitlerunning{On the Relationship between Sum-Product Networks and Bayesian Networks}

\begin{document} 

\twocolumn[
\icmltitle{On the Relationship between Sum-Product Networks and Bayesian Networks}

\icmlauthor{Han Zhao}{han.zhao@uwaterloo.ca}
\icmlauthor{Mazen Melibari}{mmelibar@uwaterloo.ca}
\icmlauthor{Pascal Poupart}{ppoupart@uwaterloo.ca}
\icmladdress{David R. Cheriton School of Computer Science, University of Waterloo, Waterloo, ON N2L 3G1, Canada}

\icmlkeywords{Sum-Product Network, Bayesian Network, Binary Decision Diagram, Algebraic Decision Diagram, Knowledge Compilation, Model Counting}

\vskip 0.3in
]

\begin{abstract} 
In this paper, we establish some theoretical connections between Sum-Product Networks (SPNs) and Bayesian Networks (BNs). We prove that every SPN can be converted into a BN in linear time and space in terms of the network size. The key insight is to use Algebraic Decision Diagrams (ADDs) to compactly represent the local conditional probability distributions at each node in the resulting BN by exploiting context-specific independence (CSI). The generated BN has a simple directed bipartite graphical structure. We show that by applying the Variable Elimination algorithm (VE) to the generated BN with ADD representations, we can recover the original SPN where the SPN can be viewed as a history record or caching of the VE inference process. To help state the proof clearly, we introduce the notion of {\em normal} SPN and present a theoretical analysis of the consistency and decomposability properties. We conclude the paper with some discussion of the implications of the proof and establish a connection between the depth of an SPN and a lower bound of the tree-width of its corresponding BN. 
\end{abstract}

\section{Introduction}
Sum-Product Networks (SPNs) have recently been proposed as tractable deep models~\cite{poon2011sum} for probabilistic inference. They distinguish themselves from other types of probabilistic graphical models (PGMs), including Bayesian Networks (BNs) and Markov Networks (MNs), by the fact that inference can be done exactly in linear time with respect to the size of the network. This has generated a lot of interest since inference is often a core task for parameter estimation and structure learning, and it typically needs to be approximated to ensure tractability since probabilistic inference in BNs and MNs is \#P-complete~\cite{roth1996hardness}.

The relationship between SPNs and BNs, and more broadly with PGMs, is not clear. Since the introduction of SPNs in the seminal paper of \citet{poon2011sum}, it is well understood that SPNs and BNs are equally expressive in the sense that they can represent any joint distribution over discrete variables\footnote{Joint distributions over continuous variables are also possible, but we will restrict ourselves to discrete variables in this paper.}, but it is not clear how to convert SPNs into BNs, nor whether a blow up may occur in the conversion process.  The common belief is that there exists a distribution such that the smallest BN that encodes this distribution is exponentially larger than the smallest SPN that encodes this same distribution. The key behind this belief lies in SPNs' ability to exploit context-specific independence (CSI)~\cite{boutilier1996context}. 

While the above belief is correct for classic BNs with tabular conditional probability distributions (CPDs) that ignore CSI, and for BNs with tree-based CPDs due to the replication problem~\cite{pagallo1989learning}, it is not clear whether it is correct for BNs with more compact representations of the CPDs. The other direction is clear for classic BNs with tabular representation: given a BN with tabular representation of its CPDs, we can build an SPN that represents the same joint probability distribution in time and space complexity that may be exponential in the tree-width of the BN. Briefly, this is done by first constructing a junction tree and translate it into an SPN\footnote{http://spn.cs.washington.edu/faq.shtml}. However, to the best of our knowledge, it is still unknown how to convert an SPN into a BN and whether the conversion will lead to a blow up when more compact representations than tables and trees are used for the CPDs.

We prove in this paper that by adopting Algebraic Decision Diagrams (ADDs)~\cite{bahar1997algebric} to represent the CPDs at each node in a BN, every SPN can be converted into a BN in linear time and space complexity in the size of the SPN. The generated BN has a simple bipartite structure, which facilitates the analysis of the structure of an SPN in terms of the structure of the generated BN. Furthermore, we show that by applying the Variable Elimination (VE) algorithm~\cite{zhang1996exploiting} to the generated BN with ADD representation of its CPDs, we can recover the original SPN in linear time and space with respect to the size of the SPN. 

Our contributions can be summarized as follows. First, we present a constructive algorithm and a proof for the conversion of SPNs into BNs using ADDs to represent the local CPDs. The conversion process is bounded by a linear function of the size of the SPN in both time and space. This gives a new perspective to understand the probabilistic semantics implied by the structure of an SPN through the generated BN. Second, we show that by executing VE on the generated BN, we can recover the original SPN in linear time and space complexity in the size of the SPN. Combined with the first point, this establishes a clear relationship between SPNs and BNs. Third, we introduce the subclass of {\em normal} SPNs and show that every SPN can be transformed into a normal SPN in quadratic time and space. Compared with general SPNs, the structure of normal SPNs exhibit more intuitive probabilistic semantics and hence normal SPNs are used as a bridge in the conversion of general SPNs to BNs. Fourth, our construction and analysis provides a new direction for learning the parameter/structure of BNs since the SPNs produced by the algorithms that learn SPNs~\cite{dennis2012learning,gens2013learning,peharz2013greedy,rooshenas2014learning} can be converted into BNs.

\section{Related Work}
Exact probabilistic reasoning has a close connection with propositional logic and weighted model counting~\cite{roth1996hardness,gomes2008model,bacchus2003algorithms,sang2005performing}. The model counting problem, \#SAT, is the problem of computing the number of models for a given propositional formula, i.e., the number of distinct truth assignments of the variables for which the formula evaluates to \texttt{TRUE}. In its weighted version, each boolean variable $X$ has a weight $\Pr(x)\in[0, 1]$ when set to \texttt{TRUE} and a weight $1-\Pr(x)$ when set to \texttt{FALSE}. The weight of a truth assignment is the product of the weights of its literals. The weighted model counting problem then asks the sum of the weights of all satisfying truth assignments. There are two important streams of research for exact weighted model counting and exact probabilistic reasoning that relate to SPNs: DPLL-style exhaustive search~\cite{birnbaum2011good} and those based on \emph{knowledge compilation}, e.g., Binary Decision Diagrams (BDDs), Decomposable Negation Normal Forms (DNNFs) and Arithmetic Circuits (ACs)~\cite{bryant1986graph,darwiche2001decomposable,darwiche2000differential}
.

The SPN, as an inference machine, has a close connection with the broader field of knowledge representation and knowledge compilation. In knowledge compilation, the reasoning process is divided into two phases: an offline compilation phase and an online query-answering phase. In the offline phase, the knowledge base, either propositional theory or belief network, is compiled into some tractable target language. In the online phase, the compiled target model is used to answer a large number of queries efficiently. The key motivation of knowledge compilation is to shift the computation that is common to many queries from the online phase into the offline phase. As an example, ACs have been studied and used extensively in both knowledge representation and probabilistic inference~\cite{darwiche2000differential,huang2006solving,chavira2006compiling}. \citet{rooshenas2014learning} recently showed that ACs and SPNs can be converted mutually without an exponential blow-up in both time and space. As a direct result, ACs and SPNs share the same expressiveness for probabilistic reasoning. 

Another representation closely related to SPNs in propositional logic and knowledge representation is the deterministic-Decomposable Negation Normal Form (d-DNNF)~\cite{darwiche2001perspective}. Propositional formulas in d-DNNF are represented by a directed acyclic graph (DAG) structure to enable the re-usability of sub-formulas.  The terminal nodes of the DAG are literals and the internal nodes are  \texttt{AND} or \texttt{OR} operators. Like SPNs, d-DNNF formulas can be queried to answer satisfiability and model counting problems.  We refer interested readers to \citet{darwiche2001perspective} and \citet{darwiche2001decomposable} for more detailed discussions.

Since their introduction by \citet{poon2011sum}, SPNs have generated a lot of interest as a tractable class of models for probabilistic inference in machine learning. Discriminative learning techniques for SPNs have been proposed and applied to image classification~\cite{gens2012discriminative}. Later, automatic structure learning algorithms were developed to build tree-structured SPNs directly from data~\cite{dennis2012learning,peharz2013greedy,gens2013learning,rooshenas2014learning}.
SPNs have also been applied to various fields and have generated promising results, including activity modeling~\cite{amer2012sum}, speech modeling~\cite{peharz2014modeling} and language modeling~\cite{cheng2014language}. Theoretical work investigating the influence of the depth of SPNs on expressiveness exists~\cite{delalleau2011shallow}, but is quite limited. As discussed later, our results reinforce previous theoretical results about the depth of SPNs and provide further insights about the structure of SPNs by examining the structure of equivalent BNs.

\section{Preliminaries}
We start by introducing the notation used in this paper. We use $1:N$ to abbreviate the notation $\{1,2,\ldots,N\}$. We use a capital letter $X$ to denote a random variable and a bold capital letter $\mathbf{X}_{1:N}$ to denote a set of random variables $\mathbf{X}_{1:N} = \{X_1, \ldots,X_N\}$. Similarly, a lowercase letter $x$ is used to denote a value taken by $X$ and a bold lowercase letter $\mathbf{x}_{1:N}$ denotes a joint value taken by the corresponding vector $\mathbf{X}_{1:N}$ of random variables. We may omit the subscript $1:N$ from $\mathbf{X}_{1:N}$ and $\mathbf{x}_{1:N}$ if it is clear from the context. For a random variable $X_i$, we use $x_i^j, j \in 1:J$ to enumerate all the values taken by $X_i$. For simplicity, we use $\Pr(x)$ to mean $\Pr(X=x)$ and $\Pr(\mathbf{x})$ to mean $\Pr(\mathbf{X} = \mathbf{x})$. We use calligraphic letters to denote graphs (e.g., $\mathcal{G}$). In particular, BNs, SPNs and ADDs are denoted respectively by $\mathcal{B}$, $\mathcal{S}$ and $\mathcal{A}$.  For a DAG $\mathcal{G}$ and a node $v$ in $\mathcal{G}$, we use $\mathcal{G}_v$ to denote the subgraph of $\mathcal{G}$ induced by $v$ and all its descendants. Let $\mathbf{V}$ be a subset of the nodes of $\mathcal{G}$, then $\mathcal{G}|_{\mathbf{V}}$ is a subgraph of $\mathcal{G}$ induced by the node set $\mathbf{V}$. Similarly, we use $\mathbf{X}|_{A}$ or $\mathbf{x}|_{A}$ to denote the restriction of a vector to a subset $A$. We use node and vertex, arc and edge interchangeably when we refer to a graph. Other notation will be introduced when needed.

To ensure that the paper is self contained, we briefly review some background material about Bayesian Networks, Algebraic Decision Diagrams and Sum-Product Networks. Readers who are already familiar with those models can skip the following subsections.

\subsection{Bayesian Network}
Consider a problem whose domain is characterized by a set of random variables $\mathbf{X}_{1:N}$ with finite support. The joint probability distribution over $\mathbf{X}_{1:N}$ can be characterized by a \emph{Bayesian Network}, which is a DAG where nodes represent the random variables and edges represent probabilistic dependencies among the variables. In a BN, we also use the terms ``node'' and ``variable'' interchangeably. For each variable in a BN, there is a local conditional probability distribution (CPD) over the variable given its parents in the BN. 

The structure of a BN encodes conditional independencies among the variables in it. Let $X_1, X_2, \ldots, X_N$ be a topological ordering of all the nodes in a BN\footnote{A topological ordering of nodes in a DAG is a linear ordering of its nodes such that each node appears after all its parents in this ordering.}, and let $\pi_{X_i}$ be the set of parents of node $X_i$ in the BN. Each variable in a BN is conditionally independent of all its non-descendants given its parents. Hence, the joint probability distribution over $\mathbf{X}_{1:N}$ admits the factorization in Eq.~\ref{equ:bnfactorization}. 
\begin{equation}
\label{equ:bnfactorization}
\Pr(\mathbf{X}_{1:N}) = \prod_{i=1}^N \Pr(X_i~|~\mathbf{X}_{1:i-1}) = \prod_{i=1}^N \Pr(X_i~|~\pi_{X_i})
\end{equation}
Given the factorization, one can use various inference algorithms to do probabilistic reasoning in BNs. See \citet{wainwright2008graphical} for a comprehensive survey.

\subsection{Algebraic Decision Diagram}
We first give a formal definition of Algebraic Decision Diagrams (ADDs) for variables with Boolean domains and then extend the definition to domains corresponding to arbitrary finite sets.
\begin{definition}[Algebraic Decision Diagram~\cite{bahar1997algebric}]
An Algebraic Decision Diagram (ADD) is a graphical representation of a real function with Boolean input variables: $f:\{0,1\}^N\mapsto\mathbb{R}$, where the graph is a rooted DAG. There are two kinds of nodes in an ADD. Terminal nodes, whose out-degree is 0, are associated with real values. Internal nodes, whose out-degree is 2, are associated with Boolean variables $X_n, n\in1:N$. For each internal node $X_n$, the left out-edge is labeled with $X_n = \texttt{FALSE}$ and the right out-edge is labeled with $X_n=\texttt{TRUE}$. 
\end{definition}
We extend the original definition of an ADD by allowing it to represent not only functions of Boolean variables, but also any function of discrete variables with a finite set as domain. This can be done by allowing each internal node $X_n$ to have $|\mathcal{X}_n|$ out-edges and label each edge with $x_n^j, j\in 1:|\mathcal{X}_n|$, where $\mathcal{X}_n$ is the domain of variable $X_n$ and $|\mathcal{X}_n|$ is the number of values $X_n$ takes. Such an ADD represents a function $f:\mathcal{X}_1\times\cdots\times\mathcal{X}_N\mapsto\mathbb{R}$,  where $\times$ means the Cartesian product between two sets. Henceforth, we will use our extended definition of ADDs throughout the paper. 

For our purpose, we will use an ADD as a compact graphical representation of local CPDs associated with each node in a BN. This is a key insight of our constructive proof presented later. Compared with a tabular representation or a decision tree representation of local CPDs, CPDs represented by ADDs can fully exploit CSI~\cite{boutilier1996context} and effectively avoid the replication problem~\cite{pagallo1989learning} of the decision tree representation. 

We give an example in Fig.~\ref{fig:add-example} where the tabular representation, decision-tree representation and ADD representation of a function of 4 Boolean variables is presented.
\begin{figure*}[htb]
\centering
	\subfigure[Tabular representation.]{
		\begin{minipage}[b]{0.3\textwidth}
			\centering
			\includegraphics[width=\textwidth]{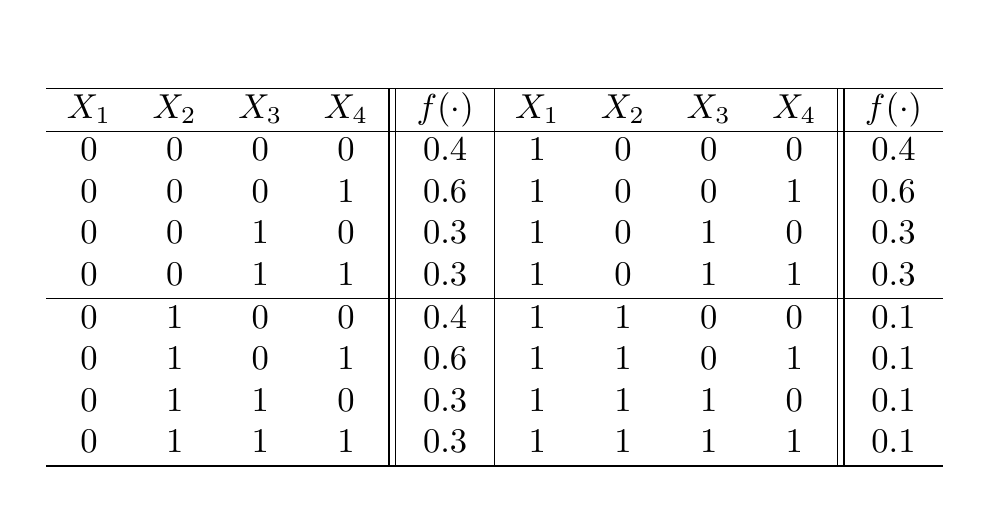}		
		\end{minipage}
		\label{fig:table-cpd}
	}
	~
	\subfigure[Decision-Tree representation.]{
		\begin{minipage}[b]{0.3\textwidth}
			\includegraphics[width=\textwidth]{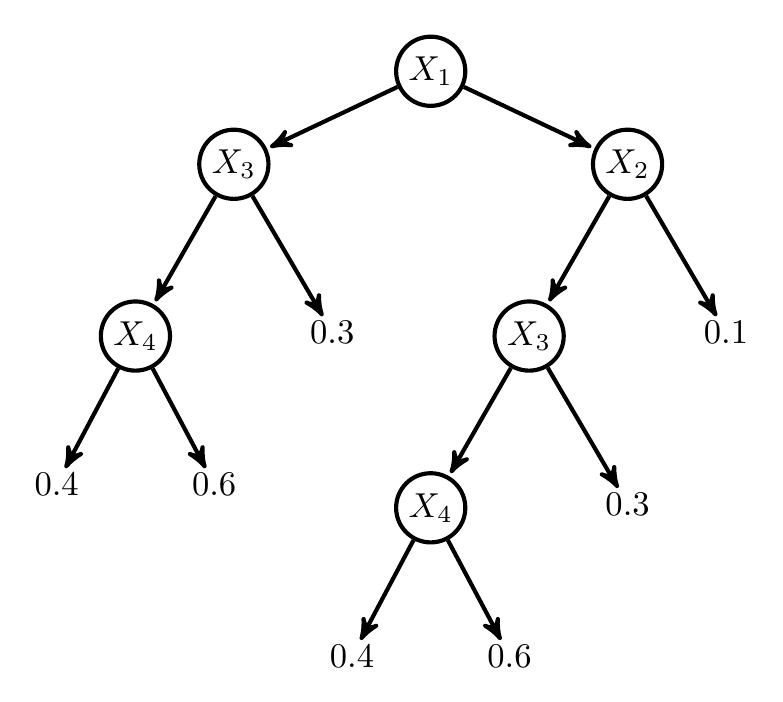}
		\end{minipage}
		\label{fig:dt-cpd}
	}
	~
	\subfigure[ADD representation.]{
		\begin{minipage}[b]{0.3\textwidth}
			\includegraphics[width=\textwidth]{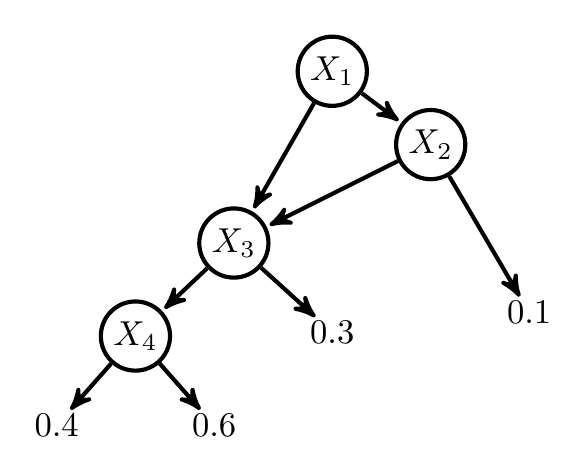}
		\end{minipage}
		\label{fig:add-cpd}
	}
\caption{Different representations of the same Boolean function. The tabular representation cannot exploit CSI and the Decision-Tree representation cannot reuse isomorphic subgraphs. The ADD representation can fully exploit CSI by sharing isomorphic subgraphs, which makes it the most compact representation among the three representations. In Fig.~\ref{fig:dt-cpd} and Fig.~\ref{fig:add-cpd}, the left and right branches of each internal node correspond respectively to \texttt{FALSE} and \texttt{TRUE}.}
\label{fig:add-example}
\end{figure*}
Another advantage of ADDs to represent local CPDs is that arithmetic operations such as multiplying ADDs and summing-out a variable from an ADD can be implemented efficiently in polynomial time.  This will allow us to use ADDs in the Variable Elimination (VE) algorithm to recover the original SPN after its conversion to a BN with CPDs represented by ADDs. Readers are referred to \citet{bahar1997algebric} for more detailed and thorough discussions about ADDs.

\subsection{Sum-Product Network}
Before introducing SPNs, we first define the notion of \emph{network polynomial}, which plays an important role in our proof. We use $\mathbb{I}[X = x]$ to denote an indicator that returns 1 when $X = x$ and 0 otherwise. To simplify the notation, we will use $\mathbb{I}_{x}$ to represent $\mathbb{I}[X = x]$.
\begin{definition}[Network Polynomial~\cite{poon2011sum}]
\label{def:networkpolynomial}
Let $f(\cdot)\geq 0$ be an unnormalized probability distribution over a Boolean random vector $\mathbf{X}_{1:N}$. The network polynomial of $f(\cdot)$ is a multilinear function $\sum_{\mathbf{x}}f(\mathbf{x})\prod_{n=1}^N\mathbb{I}_{\mathbf{x}_n}$ of indicator variables, where the summation is over all possible instantiations of the Boolean random vector $\mathbf{X}_{1:N}$.
\end{definition}
Intuitively, the network polynomial is a Boolean expansion~\cite{boole1847mathematical} of the unnormalized probability distribution $f(\cdot)$. For example, the network polynomial of a BN $X_1\rightarrow X_2$ is $\Pr(x_1, x_2)\mathbb{I}_{x_1}\mathbb{I}_{x_2} + \Pr(x_1, \bar{x}_2)\mathbb{I}_{x_1}\mathbb{I}_{\bar{x}_2} + \Pr(\bar{x}_1, x_2)\mathbb{I}_{\bar{x}_1}\mathbb{I}_{x_2} + \Pr(\bar{x}_1, \bar{x}_2)\mathbb{I}_{\bar{x}_1}\mathbb{I}_{\bar{x}_2}$.
\begin{definition}[Sum-Product Network~\cite{poon2011sum}]
\label{def:spn}
A Sum-Product Network (SPN) over Boolean variables $\mathbf{X}_{1:N}$ is a rooted DAG whose leaves are the indicators $\mathbb{I}_{x_1},\ldots,\mathbb{I}_{x_N}$ and $\mathbb{I}_{\bar{x}_1}, \ldots, \mathbb{I}_{\bar{x}_N}$ and whose internal nodes are sums and products. Each edge $(v_i,v_j)$ emanating from a sum node $v_i$ has a non-negative weight $w_{ij}$. The value of a product node is the product of the values of its children. The value of a sum node is $\sum_{v_j\in Ch(v_i)}w_{ij}val(v_j)$ 
where $Ch(v_i)$ are the children of $v_i$ and 
$val(v_j)$ is the value of node $v_j$. The value of an SPN $\mathcal{S}[\mathbb{I}_{x_1}, \mathbb{I}_{\bar{x}_1}, \ldots, \mathbb{I}_{x_N}, \mathbb{I}_{\bar{x}_N}]$ is the value of its root.
\end{definition}
The \emph{scope} of a node in an SPN is defined as the set of variables that have indicators among the node's descendants: For any node $v$ in an SPN, if $v$ is a terminal node, say, an indicator variable over $X$, then $\text{scope}(v)=\{X\}$, else $\text{scope}(v)=\bigcup_{\tilde{v}\in Ch(v)}\text{scope}(\tilde{v})$. \citet{poon2011sum} further define the following properties of an SPN:
\begin{definition}[Complete]
An SPN is \emph{complete} iff each sum node has children with the same scope.
\end{definition}
\begin{definition}[Consistent]
An SPN is consistent iff no variable appears negated in one child of a product node and non-negated in another.
\end{definition}
\begin{definition}[Decomposable]
An SPN is decomposable iff for every product node $v$, scope($v_i$) $\bigcap$ scope($v_j$) $=\varnothing$ where $v_i, v_j\in Ch(v), i\neq j$. 
\end{definition}
Clearly, decomposability implies consistency in SPNs. An SPN is said to be \emph{valid} iff it defines a (unnormalized) probability distribution. \citet{poon2011sum} proved that if an SPN is complete and consistent, then it is valid. Note that this is a sufficient, but not necessary condition. In this paper, we focus only on complete and consistent SPNs as we are interested in their associated probabilistic semantics. For a complete and consistent SPN $\mathcal{S}$, each node $v$ in $\mathcal{S}$ defines a network polynomial $f_v(\cdot)$ which corresponds to the sub-SPN rooted at $v$. The network polynomial defined by the root of the SPN can then be computed recursively by taking a weighted sum of the network polynomials defined by the sub-SPNs rooted at the children of each sum node and a product of the network polynomials defined by the sub-SPNs rooted at the children of each product node. The probability distribution induced by an SPN $\mathcal{S}$ is defined as $\Pr_{\mathcal{S}}(\mathbf{x})\triangleq \frac{f_S(\mathbf{x})}{\sum_{\mathbf{x}}f_S(\mathbf{x})}$, where $f_{\mathcal{S}}(\cdot)$ is the network polynomial defined by the root of the SPN $\mathcal{S}$. An example of a complete and consistent SPN is given in Fig.~\ref{fig:spn-example}. 
\begin{figure}[htb]
\centering
	\includegraphics[width=0.6\linewidth]{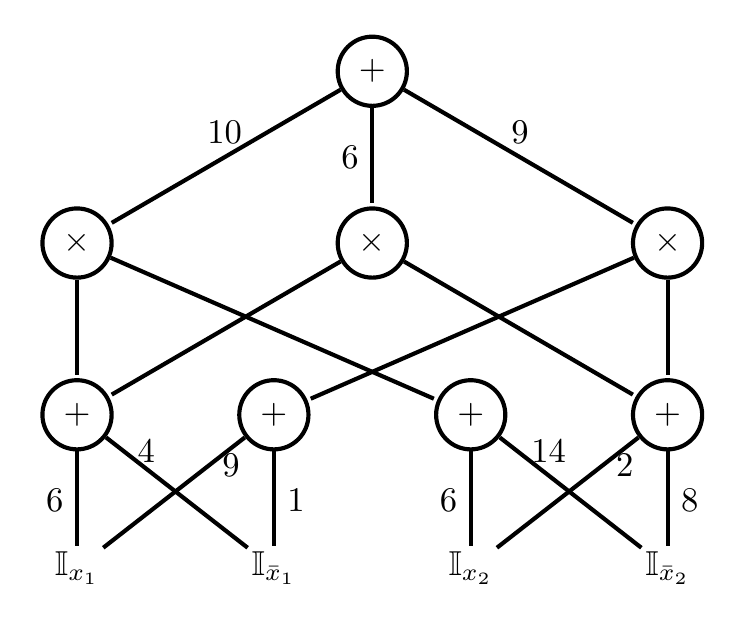}
	\caption{A complete and consistent SPN over Boolean variables $X_1$, $X_2$. This SPN is also decomposable since every product node has children whose scopes do not intersect. The network polynomial defined by (the root of) this SPN is: $f(X_1, X_2) = 10(6\mathbb{I}_{x_1} + 4\mathbb{I}_{\bar{x}_1})(6\mathbb{I}_{x_2} + 14\mathbb{I}_{\bar{x}_2}) + 6(6\mathbb{I}_{x_1} + 4\mathbb{I}_{\bar{x}_1})(2\mathbb{I}_{x_2} + 8\mathbb{I}_{\bar{x}_2}) + 9(9\mathbb{I}_{x_1} + \mathbb{I}_{\bar{x}_1})(2\mathbb{I}_{x_2} + 8\mathbb{I}_{\bar{x}_2}) = 594\mathbb{I}_{x_1}\mathbb{I}_{x_2} + 1776\mathbb{I}_{x_1}\mathbb{I}_{\bar{x}_2} + 306\mathbb{I}_{\bar{x}_1}\mathbb{I}_{x_2} + 824\mathbb{I}_{\bar{x}_1}\mathbb{I}_{\bar{x}_2}$ and the probability distribution induced by $\mathcal{S}$ is $\Pr_{\mathcal{S}} = \frac{594}{3500}\mathbb{I}_{x_1}\mathbb{I}_{x_2} + \frac{1776}{3500}\mathbb{I}_{x_1}\mathbb{I}_{\bar{x}_2} + \frac{306}{3500}\mathbb{I}_{\bar{x}_1}\mathbb{I}_{x_2} + \frac{824}{3500}\mathbb{I}_{\bar{x}_1}\mathbb{I}_{\bar{x}_2}$.}
	\label{fig:spn-example}
\end{figure}

\section{Main Results}
In this section, we first state the main results obtained in this paper and then provide detailed proofs with some discussion of the results. To keep the presentation simple, we assume without loss of generality that all the random variables are Boolean unless explicitly stated. It is straightforward to extend our analysis to discrete random variables with finite support. For an SPN $\mathcal{S}$, let $|\mathcal{S}|$ be the size of the SPN, i.e., the number of nodes plus the number of edges in the graph. For a BN $\mathcal{B}$, the size of $\mathcal{B}$, $|\mathcal{B}|$, is defined by the size of the graph \emph{plus} the size of all the CPDs in $\mathcal{B}$ (the size of a CPD depends on its representation, which will be clear from the context). The main theorems are:
\begin{theorem}
\label{thm:spn2bn}
There exists an algorithm that converts any complete and decomposable SPN $\mathcal{S}$ over Boolean variables $\mathbf{X}_{1:N}$ into a BN $\mathcal{B}$ with CPDs represented by ADDs in time $O(N|\mathcal{S}|)$.  Furthermore, $\mathcal{S}$ and $\mathcal{B}$ represent the same distribution and $|\mathcal{B}| = O(N|\mathcal{S}|)$.
\end{theorem}
As it will be clear later, Thm.~\ref{thm:spn2bn} immediately leads to the following corollary:
\begin{corollary}
\label{coro:spn2bn}
There exists an algorithm that converts any complete and consistent SPN $\mathcal{S}$ over Boolean variables $\mathbf{X}_{1:N}$ into a BN $\mathcal{B}$ with CPDs represented by ADDs in time $O(N|\mathcal{S}|^2)$.  Furthermore, $\mathcal{S}$ and $\mathcal{B}$ represent the same distribution and $|\mathcal{B}| = O(N|\mathcal{S}|^2)$.
\end{corollary}
\begin{remark}
The BN $\mathcal{B}$ generated from $\mathcal{S}$ in Theorem~\ref{thm:spn2bn} and Corollary~\ref{coro:spn2bn} has a simple bipartite DAG structure, where all the source nodes are hidden variables and the terminal nodes are the Boolean variables $\mathbf{X}_{1:N}$.
\end{remark}
\begin{remark}
Assuming sum nodes alternate with product nodes in SPN $\mathcal{S}$, the depth of $\mathcal{S}$ is proportional to the maximum in-degree of the nodes in $\mathcal{B}$, which, as a result, is proportional to a lower bound of the tree-width of $\mathcal{B}$.
\end{remark}
\begin{theorem}
\label{thm:bn2spn}
Given the BN $\mathcal{B}$ with ADD representation of CPDs generated from a complete and decomposable SPN $\mathcal{S}$ over Boolean variables $\mathbf{X}_{1:N}$, the original SPN $\mathcal{S}$ can be recovered by applying the Variable Elimination algorithm to $\mathcal{B}$ in $O(N|\mathcal{S}|)$.
\end{theorem}
\begin{remark}
The combination of Theorems~\ref{thm:spn2bn} and~\ref{thm:bn2spn} shows that distributions for which SPNs allow a compact representation and efficient inference, BNs with ADDs also allow a compact representation and efficient inference (i.e., no exponential blow up). 
\end{remark}
To make the upcoming proofs concise, we first define a \emph{normal form} for SPNs and show that every complete and consistent SPN can be transformed into a normal SPN in quadratic time and space without changing the network polynomial. We then derive the proofs with normal SPNs. Note that we only focus on SPNs that are \emph{complete} and \emph{consistent}. Hence, when we refer to an SPN, we assume that it is complete and consistent without explicitly stating this.

\subsection{Normal Form}
For an SPN $\mathcal{S}$, let $f_{\mathcal{S}}(\cdot)$ be the network polynomial defined at the root of $\mathcal{S}$. Define the \emph{height} of an SPN to be the length of the longest path from the root to a terminal node. 
\begin{definition}
\label{def:normal}
An SPN is said to be normal if
\begin{enumerate}
	\item 	It is complete and decomposable.
	\item 	For each sum node in the SPN, the weights of the edges emanating from the sum node are nonnegative and sum to 1.
	\item 	Every terminal node in the SPN is a univariate distribution over a Boolean variable and the size of the scope of a sum node is at least 2 (sum nodes whose scope is of size 1 are reduced into terminal nodes).
\end{enumerate}
\end{definition}
\begin{theorem}
\label{thm:normal}
For any complete and consistent SPN $\mathcal{S}$, there exists a normal SPN $\mathcal{S}'$ such that $\Pr_{\mathcal{S}}(\cdot) = \Pr_{\mathcal{S}'}(\cdot)$ and $|\mathcal{S}'| = O(|\mathcal{S}|^2)$.
\end{theorem}
To show this, we first prove the following lemmas.
\begin{lemma}
\label{lemma:decomposability}
For any complete and consistent SPN $\mathcal{S}$ over $\mathbf{X}_{1:N}$, there exists a complete and decomposable SPN $\mathcal{S}'$ over $\mathbf{X}_{1:N}$ such that $f_{\mathcal{S}}(\mathbf{x}) = f_{\mathcal{S}'}(\mathbf{x}), \forall\mathbf{x}$ and $|\mathcal{S}'| = O(|\mathcal{S}|^2)$.
\end{lemma}
\begin{proof}
Let $\mathcal{S}$ be a complete and consistent SPN. If it is also decomposable, then simply set $\mathcal{S}' = \mathcal{S}$ and we are done. Otherwise, let $v_1, \ldots, v_M$ be an inverse topological ordering of all the nodes in $\mathcal{S}$, including both terminal nodes and internal nodes, such that for any $v_m, m\in 1:M$, all the ancestors of $v_m$ in the graph appear after $v_m$ in the ordering. Let $v_m$ be the first product node in the ordering that violates decomposability. Let $v_{m_1}, v_{m_2}, \ldots, v_{m_l}$ be the children of $v_m$ where $m_1< m_2<\cdots < m_l < m$ (due to the inverse topological ordering). Let $(v_{m_i}, v_{m_j}), i< j, i, j\in 1:l$ be the first ordered pair of nodes such that $\text{scope}(v_{m_i})\bigcap\text{scope}(v_{m_j})\neq\varnothing$. Hence, let $X\in\text{scope}(v_{m_i})\bigcap\text{scope}(v_{m_j})$. Consider $f_{v_{m_i}}$ and $f_{v_{m_j}}$ which are the network polynomials defined by the sub-SPNs rooted at $v_{m_i}$ and $v_{m_j}$. 

Expand network polynomials $f_{v_{m_i}}$ and $f_{v_{m_j}}$ into a \emph{sum-of-product} form by applying the distributive law between products and sums. For example, if $f(X_1, X_2) = (\mathbb{I}_{x_1} + 9\mathbb{I}_{\bar{x}_1})(4\mathbb{I}_{x_2} + 6\mathbb{I}_{\bar{x}_2})$, then the expansion of $f$ is $f(X_1, X_2) = 4\mathbb{I}_{x_1}\mathbb{I}_{x_2} + 6\mathbb{I}_{x_1}\mathbb{I}_{\bar{x}_2} + 36\mathbb{I}_{\bar{x}_1}\mathbb{I}_{x_2} + 54\mathbb{I}_{\bar{x}_1}\mathbb{I}_{\bar{x}_2}$. Since $\mathcal{S}$ is complete, then sub-SPNs rooted at $v_{m_i}$ and $v_{m_j}$ are also complete, which means that each monomial in the expansion of $f_{v_{m_i}}$ must share the same scope.  The same applies to $f_{v_{m_j}}$. Since $X\in\text{scope}(v_{m_i})\bigcap\text{scope}(v_{m_j})$, then every monomial in the expansion of $f_{v_{m_i}}$ and $f_{v_{m_j}}$ must contain an indicator variable over $X$, either $\mathbb{I}_{x}$ or $\mathbb{I}_{\bar{x}}$. Furthermore, since $\mathcal{S}$ is consistent, then the sub-SPN rooted at $v_m$ is also consistent. Consider $f_{v_m} = \prod_{k = 1}^l f_{v_{m_k}} = f_{v_{m_i}}f_{v_{m_j}}\prod_{k\neq i,j}f_{v_{m_k}}$. Because $v_m$ is consistent, we know that each monomial in the expansions of $f_{v_{m_i}}$ and $f_{v_{m_j}}$ must contain the same indicator variable of $X$, either $\mathbb{I}_{x}$ or $\mathbb{I}_{\bar{x}}$, otherwise there will be a term $\mathbb{I}_{x}\mathbb{I}_{\bar{x}}$ in $f_{v_m}$ which violates the consistency assumption. Without loss of generality, assume each monomial in the expansions of $f_{v_{m_i}}$ and $f_{v_{m_j}}$ contains $\mathbb{I}_x$. Then we can re-factorize $f_{v_m}$ in the following way:
\begin{align}
f_{v_m} & = \prod_{k = 1}^l f_{v_{m_k}} = \mathbb{I}^2_x \frac{f_{v_{m_i}}}{\mathbb{I}_x} \frac{f_{v_{m_j}}}{\mathbb{I}_x}\prod_{k\neq i,j}f_{v_{m_k}}\nonumber\\
& = \mathbb{I}_x\frac{f_{v_{m_i}}}{\mathbb{I}_x} \frac{f_{v_{m_j}}}{\mathbb{I}_x}\prod_{k\neq i,j}f_{v_{m_k}} = \mathbb{I}_x\tilde{f}_{v_{m_i}} \tilde{f}_{v_{m_j}}\prod_{k\neq i,j}f_{v_{m_k}}\label{equ:extractout}
\end{align}
where we use the fact that indicator variables are idempotent, i.e., $\mathbb{I}^2_{x} = \mathbb{I}_{x}$ and $\tilde{f}_{v_{m_i}}(\tilde{f}_{v_{m_j}})$ is defined as the function by factorizing $\mathbb{I}_{x}$ out from $f_{v_{m_i}}(f_{v_{m_j}})$.  Eq.~\ref{equ:extractout} means that in order to make $v_m$ decomposable, we can simply remove all the indicator variables $\mathbb{I}_x$ from sub-SPNs rooted at $v_{m_i}$ and $v_{m_j}$ and later link $\mathbb{I}_x$ to $v_m$ directly. Such a transformation will not change the network polynomial $f_{v_m}$ as shown by Eq.~\ref{equ:extractout}, but it will remove $X$ from $\text{scope}(v_{m_i})\bigcap\text{scope}(v_{m_j})$. In principle, we can apply this transformation to all ordered pairs $(v_{m_i}, v_{m_j}), i < j, i, j \in 1:l$ with nonempty intersections of scope. However, this is not algorithmically efficient and more importantly, for local components containing $\mathbb{I}_x$ in $f_{v_m}$ which are reused by other nodes $v_n$ outside of $\mathcal{S}_{v_m}$, we cannot remove $\mathbb{I}_x$ from them otherwise the network polynomials for each such $v_n$ will be changed due to the removal. In such case, we need to duplicate the local components to ensure that local transformations with respect to $f_{v_m}$ do not affect network polynomials $f_{v_n}$. We present the transformation in Alg.~\ref{alg:transformation}.
\begin{algorithm}[hbt]
\centering
\caption{Decomposition Transformation}
\label{alg:transformation}
\begin{algorithmic}[1]
\REQUIRE	Complete and consistent SPN $\mathcal{S}$.
\ENSURE	Complete and decomposable SPN $\mathcal{S}'$.
\STATE	Let $v_1, v_2, \ldots, v_M$ be an inverse topological ordering of nodes in $\mathcal{S}$.
\FOR {$m = 1$ to $M$}
	\IF {$v_m$ is a non-decomposable product node}
		\STATE	$\Omega(v_m)\leftarrow \bigcup_{i\neq j}\text{scope}(v_{m_i})\bigcap\text{scope}(v_{m_j})$
		\STATE	$\mathbf{V}\leftarrow\{v\in\mathcal{S}_{v_m}~|~\text{scope}(v)\bigcap\Omega(v_m)\neq\varnothing\}$
		\STATE	$\mathcal{S}_{\mathbf{V}}\leftarrow \mathcal{S}_{v_m}|_{\mathbf{V}}$
		\STATE	$D(v_m)\leftarrow$ descendants of $v_m$
		\FOR {node $v\in\mathcal{S}_{\mathbf{V}}\backslash\{v_m\}$}
			\IF {$Pa(v)\backslash D(v_m)\neq \varnothing$}
				\STATE	Create $p\leftarrow v\otimes \prod_{X\in\Omega(v_m)\cap \text{scope}(v)}\mathbb{I}_{x^*}$
				\STATE	Connect $p$ to $\forall f\in Pa(v)\backslash D(v_m)$
				\STATE	Disconnect $v$ from $\forall f\in Pa(v)\backslash D(v_m)$
			\ENDIF
		\ENDFOR
		\FOR {node $v\in\mathcal{S}_{\mathbf{V}}$ in bottom-up order}
			\STATE	Disconnect $\tilde{v}\in Ch(v)$ $\forall \text{scope}(\tilde{v})\subseteq \Omega(v_m)$ 
		\ENDFOR
		\STATE	Connect $\prod_{X\in\Omega(v_m)}\mathbb{I}_{x^*}$ to $v_m$ directly
	\ENDIF
\ENDFOR
\STATE	Delete all nodes unreachable from the root of $\mathcal{S}$
\STATE	Delete all product nodes with out-degree 0
\STATE	Contract all product nodes with out-degree 1
\end{algorithmic}
\end{algorithm}
Alg.~\ref{alg:transformation} transforms a complete and consistent SPN $\mathcal{S}$ into a complete and decomposable SPN $\mathcal{S}'$. Informally, it works using the following identity:
\begin{equation}
\label{equ:transformation}
f_{v_m}
= \left(\prod_{X\in\Omega(v_m)}\mathbb{I}_{x^*}\right)\prod_{k=1}^l \frac{f_{v_{m_k}}}{\prod_{X\in\Omega(v_m)\cap\text{scope}(v_{m_k})}\mathbb{I}_x^*}
\end{equation}
where $\Omega(v_m)\triangleq\bigcup_{i,j\in 1:l, i\neq j}\text{scope}(v_{m_i})\cap\text{scope}(v_{m_j})$, i.e., $\Omega(v_m)$ is the union of all the shared variables between pairs of children of $v_m$ and $\mathbb{I}_{x^*}$ is the indicator variable of $X\in\Omega(v_m)$ appearing in $\mathcal{S}_{v_m}$. Based on the analysis above, we know that for each $X\in\Omega(v_m)$ there will be only one kind of indicator variable $\mathbb{I}_{x^*}$ that appears inside $\mathcal{S}_{v_m}$, otherwise $v_m$ is not consistent. In Line 6, $\mathcal{S}_{v_m}|_{\mathbf{V}}$ is defined as the sub-SPN of $\mathcal{S}_{v_m}$ induced by the node set $\mathbf{V}$, i.e., a subgraph of $\mathcal{S}_{v_m}$ where the node set is restricted to $\mathbf{V}$. In Lines 5-6, we first extract the induced sub-SPN $\mathcal{S}_{\mathbf{V}}$ from $\mathcal{S}_{v_m}$ rooted at $v_m$ using the node set in which nodes have nonempty intersections with $\Omega(v_m)$. We disconnect the nodes in $\mathcal{S}_{\mathbf{V}}$ from their children if their children are indicator variables of a subset of $\Omega(v_m)$ (Lines 15-17). At Line 18, we build a new product node by multiplying all the indicator variables in $\Omega(v_m)$ and link it to $v_m$ directly. To keep unchanged the network polynomials of nodes outside $\mathcal{S}_{v_m}$ that use nodes in $\mathcal{S}_{\mathbf{V}}$, we create a duplicate node $p$ for each such node $v$ and link $p$ to all the parents of $v$ outside of $\mathcal{S}_{v_m}$ and at the same time delete the original link (Lines 9-13).

In summary, Lines 15-17 ensure that $v_m$ is decomposable by removing all the shared indicator variables in $\Omega(v_m)$. Line 18 together with Eq.~\ref{equ:transformation} guarantee that $f_{v_m}$ is unchanged after the transformation. Lines 9-13 create necessary duplicates to ensure that other network polynomials are not affected. Lines 21-23 simplify the transformed SPN to make it more compact. An example is depicted in Fig.~\ref{fig:transformation} to illustrate the transformation process.
\begin{figure}[htb]
\centering
    \subfigure {
        \begin{minipage}[b]{0.45\linewidth}
            \centering
            \includegraphics[width=\textwidth]{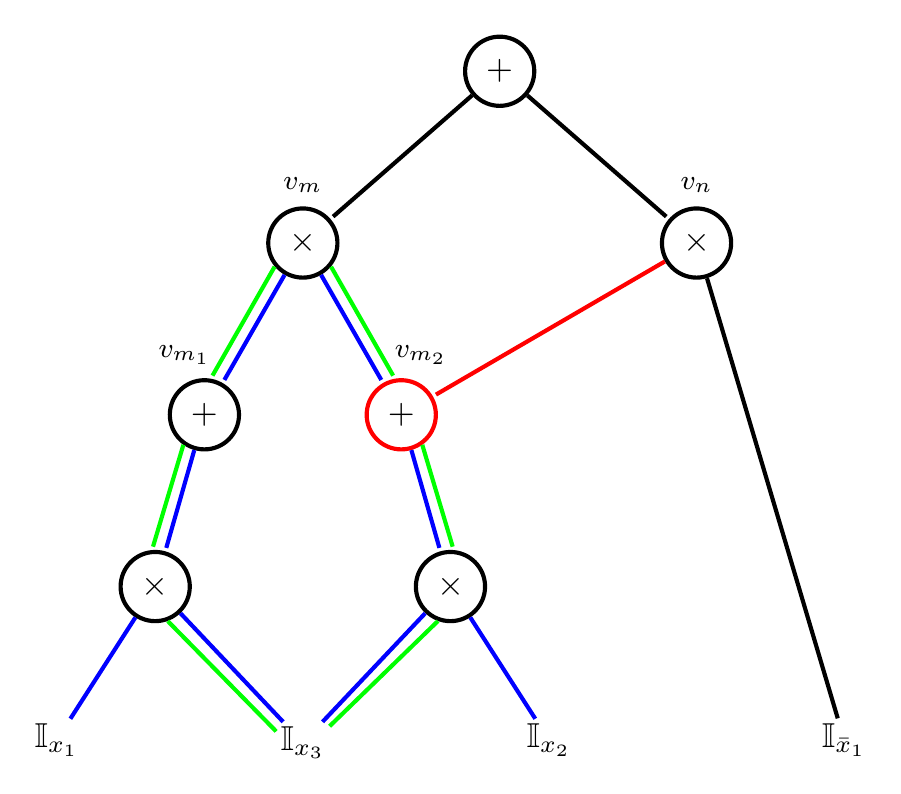}     
        \end{minipage}
        \label{fig:transform-a}
    }
    ~
    \subfigure {
        \begin{minipage}[b]{0.45\linewidth}
            \includegraphics[width=\textwidth]{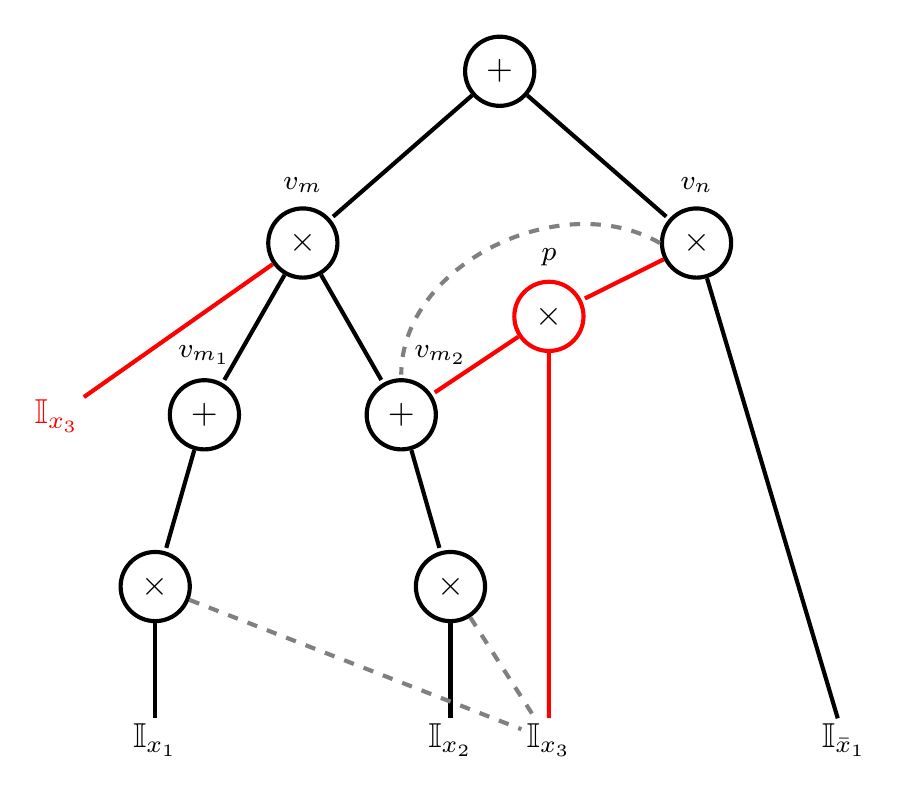}
        \end{minipage}
        \label{fig:transform-b}
    }
\caption{Transformation process described in Alg.~\ref{alg:transformation} to construct a complete and decomposable SPN from a complete and consistent SPN. The product node $v_m$ in the left SPN is not decomposable. Induced sub-SPN $\mathcal{S}_{v_m}$ is highlighted in blue and $\mathcal{S}_{\mathbf{V}}$ is highlighted in green. $v_{m_2}$ highlighted in red is reused by $v_n$ which is outside $\mathcal{S}_{v_m}$. To compensate for $v_{m_2}$, we create a new product node $p$ in the right SPN and connect it to indicator variable $\mathbb{I}_{x_3}$ and $v_{m_2}$. Dashed gray lines in the right SPN denote deleted edges and nodes while red edges and nodes are added during Alg.~\ref{alg:transformation}.}
\label{fig:transformation}
\end{figure}

We now analyze the size of the SPN constructed by Alg.~\ref{alg:transformation}. For a graph $\mathcal{S}$, let $\mathfrak{V}(\mathcal{S})$ be the number of nodes in $\mathcal{S}$ and let $\mathfrak{E}(\mathcal{S})$ be the number of edges in $\mathcal{S}$. Note that in Lines 8-17 we only focus on nodes that appear in the induced SPN $\mathcal{S}_{\mathbf{V}}$, which clearly has $|\mathcal{S}_{\mathbf{V}}| \leq |\mathcal{S}_{v_m}|$. Furthermore, we create a new product node $p$ at Line 10 iff $v$ is reused by other nodes which do not appear in $\mathcal{S}_{v_m}$. This means that the number of nodes created during each iteration between Lines 2 and 20 is bounded by $\mathfrak{V}(\mathcal{S}_{\mathbf{V}})\leq \mathfrak{V}(\mathcal{S}_{v_m})$. Line 10 also creates 2 new edges to connect $p$ to $v$ and the indicator variables. Lines 11 and 12 first connect edges to $p$ and then delete edges from $v$, hence these two steps do not yield increases in the number of edges. So the increase in the number of edges is bounded by $2\mathfrak{V}(\mathcal{S}_{\mathbf{V}})\leq 2\mathfrak{V}(\mathcal{S}_{v_m})$. Combining increases in both nodes and edges, during each outer iteration the increase in size is bounded by $3|\mathcal{S}_{\mathbf{V}}| \leq 3|\mathcal{S}_{v_m}| = O(|\mathcal{S}|)$. There will be at most $M = \mathfrak{V}(\mathcal{S})$ outer iterations hence the total increase in size will be bounded by $O(M|\mathcal{S}|) = O(|\mathcal{S}|^2)$.
\end{proof}
\begin{lemma}
\label{lemma:one}
For any complete and decomposable SPN $\mathcal{S}$ over $\mathbf{X}_{1:N}$ that satisfies condition 2 of Def.~\ref{def:normal}, $\sum_{\mathbf{x}}f_{\mathcal{S}}(\mathbf{x}) = 1$.
\end{lemma}
\begin{proof}
We give a proof by induction on the height of $\mathcal{S}$. Let $R$ be the root of $\mathcal{S}$.
\begin{itemize}
	\item 	Base case. SPNs of height 0 are indicator variables over some Boolean variable whose network polynomials immediately satisfy Lemma~\ref{lemma:one}. 
	\item 	Induction step. Assume Lemma~\ref{lemma:one} holds for any SPN with height $\leq k$. Consider an SPN $\mathcal{S}$ with height $k+1$. We consider the following two cases:
		\begin{itemize}
			\item 	The root $R$ of $\mathcal{S}$ is a product node. Then in this case the network polynomial $f_{\mathcal{S}}(\cdot)$ for $\mathcal{S}$ is defined as $f_{\mathcal{S}} = \prod_{v\in Ch(R)}f_v$. We have 
\begin{align}
\sum_{\mathbf{x}}f_{\mathcal{S}}(\mathbf{x}) & = \sum_{\mathbf{x}}\prod_{v\in Ch(R)}f_v(\mathbf{x}|_{\text{scope}(v)}) \label{equ:h}\\
& = \prod_{v\in Ch(R)}\sum_{\mathbf{x}|_{\text{scope}(v)}}f_{v}(\mathbf{x}|_{\text{scope}(v)}) \label{equ:h2}\\ 
& = \prod_{v\in Ch(R)}1 = 1 \label{equ:h3}
\end{align}
where $\mathbf{x}|_{\text{scope}(v)}$ means that $\mathbf{x}$ is restricted to the set $\text{scope}(v)$. Eq.~\ref{equ:h2} follows from the decomposability of $R$ and Eq.~\ref{equ:h3} follows from the induction hypothesis. 
			\item 	The root $R$ of $\mathcal{S}$ is a sum node. The network polynomial is $f_{\mathcal{S}} = \sum_{v\in Ch(R)}w_{R, v}f_{v}$. We have 
\begin{align}
\sum_{\mathbf{x}} f_{\mathcal{S}}(\mathbf{x}) & = \sum_{\mathbf{x}}\sum_{v\in Ch(R)}w_{R, v}f_{v}(\mathbf{x}) \label{equ:g} \\
& = \sum_{v\in Ch(R)}w_{R, v}\sum_{\mathbf{x}}f_{v}(\mathbf{x}) \label{equ:g2} \\
& = \sum_{v\in Ch(R)}w_{R, v} = 1 \label{equ:g3}
\end{align}
Eq.~\ref{equ:g2} follows from the commutative and associative law of addition and Eq.~\ref{equ:g3} follows by the induction hypothesis. 
		\end{itemize}
\end{itemize}
\end{proof}
\begin{corollary}
\label{coro:equal}
For any complete and decomposable SPN $\mathcal{S}$ over $\mathbf{X}_{1:N}$ that satisfies condition 2 of Def.~\ref{def:normal}, $\Pr_{\mathcal{S}}(\cdot) = f_{\mathcal{S}}(\cdot)$.
\end{corollary}

\begin{lemma}
\label{lemma:normalization}
For any complete and decomposable SPN $\mathcal{S}$, there exists an SPN $\mathcal{S}'$ where the weights of the edges emanating from every sum node are nonnegative and sum to 1, and $\Pr_{\mathcal{S}}(\cdot) = \Pr_{\mathcal{S}'}(\cdot)$, $|\mathcal{S}'| = |\mathcal{S}|$.
\end{lemma}
\begin{proof}
Alg.~\ref{alg:normalization} runs in one pass of $\mathcal{S}$ to construct the required SPN $\mathcal{S}'$.
\begin{algorithm}[htb]
\centering
\caption{Weight Normalization}
\label{alg:normalization}
\begin{algorithmic}[1]
\REQUIRE	SPN $\mathcal{S}$
\ENSURE	SPN $\mathcal{S}'$
\STATE	$\mathcal{S}'\leftarrow \mathcal{S}$
\STATE	 $val(\mathbb{I}_x)\leftarrow 1, \forall\mathbb{I}_x\in\mathcal{S}$
\STATE	Let $v_1, \ldots, v_M$ be an inverse topological ordering of the nodes in $\mathcal{S}$
\FOR {$m = 1$ to $M$}
	\IF {$v_m$ is a sum node}
		\STATE	$val(v_m)\leftarrow \sum_{v\in Ch(v_m)}w_{v_m, v}val(v)$
		\STATE	$w'_{v_m, v}\leftarrow \frac{w_{v_m, v}val(v)}{val(v_m)}, \quad\forall v\in Ch(v_m)$
	\ELSIF {$v_m$ is a product node}
		\STATE	$val(v_m)\leftarrow \prod_{v\in Ch(v_m)}val(v)$	
	\ENDIF
\ENDFOR
\end{algorithmic}
\end{algorithm}
We proceed to prove that the SPN $\mathcal{S}'$ returned by Alg.~\ref{alg:normalization} satisfies $\Pr_{\mathcal{S}'}(\cdot) = \Pr_{\mathcal{S}}(\cdot)$, $|\mathcal{S}'| = |\mathcal{S}|$ and that $\mathcal{S}'$ satisfies condition 2 of Def.~\ref{def:normal}. It is clear that $|\mathcal{S}'| = |\mathcal{S}|$ because we only modify the weights of $\mathcal{S}$ to construct $\mathcal{S}'$ at Line 7. Based on Lines 6 and 7, it is also straightforward to verify that for each sum node $v$ in $\mathcal{S}'$, the weights of the edges emanating from $v$ are nonnegative and sum to 1. We now show that $\Pr_{\mathcal{S}'}(\cdot) = \Pr_{\mathcal{S}}(\cdot)$. Using Corollary~\ref{coro:equal}, $\Pr_{\mathcal{S}'}(\cdot) = f_{\mathcal{S}'}(\cdot)$. Hence it is sufficient to show that $f_{\mathcal{S}'}(\cdot) = \Pr_{\mathcal{S}}(\cdot)$. Before deriving a proof, it is helpful to note that for each node $v\in\mathcal{S}$, $val(v) = \sum_{\mathbf{x}|_{\text{scope}(v)}}f_v(\mathbf{x}|_{\text{scope}(v)})$. We give a proof by induction on the height of $\mathcal{S}$.
\begin{itemize}
	\item 	Base case. SPNs with height 0 are indicator variables which automatically satisfy Lemma~\ref{lemma:normalization}. 
	\item 	Induction step. Assume Lemma~\ref{lemma:normalization} holds for any SPN of height $\leq k$. Consider an SPN $\mathcal{S}$ of height $k+1$. Let $R$ be the root node of $\mathcal{S}$ with out-degree $l$. We discuss the following two cases.
	\begin{itemize}
		\item 	$R$ is a product node. Let $R_1,\ldots, R_{l}$ be the children of $R$ and $\mathcal{S}_{1},\ldots, \mathcal{S}_{l}$ be the corresponding sub-SPNs. By induction, Alg.~\ref{alg:normalization} returns $\mathcal{S}'_{1},\ldots,\mathcal{S}'_{l}$ that satisfy Lemma~\ref{lemma:normalization}. Since $R$ is a product node, we have
\begin{align}
f_{\mathcal{S}'}(\mathbf{x}) & = \prod_{i=1}^l f_{\mathcal{S}'_i}(\mathbf{x}|_{\text{scope}(R_i)}) \\ 
& = \prod_{i=1}^l\Pr_{\mathcal{S}_i}(\mathbf{x}|_{\text{scope}(R_i)}) \label{equ:9}\\
& = \prod_{i=1}^l \frac{f_{\mathcal{S}_i}(\mathbf{x}|_{\text{scope}(R_i)})}{\sum_{\mathbf{x}|_{\text{scope}(R_i)}}f_{\mathcal{S}_i}(\mathbf{x}|_{\text{scope}(R_i)})} \label{equ:10}\\
& = \frac{\prod_{i=1}^l f_{\mathcal{S}_i}(\mathbf{x}|_{\text{scope}(R_i)})}{\sum_{\mathbf{x}}\prod_{i=1}^l f_{\mathcal{S}_i}(\mathbf{x}|_{\text{scope}(R_i)})} \label{equ:11}\\
& = \frac{f_{\mathcal{S}}(\mathbf{x})}{\sum_{\mathbf{x}}f_{\mathcal{S}}(\mathbf{x})} = \Pr_{\mathcal{S}}(\mathbf{x}) \label{equ:12}
\end{align}
Eq.~\ref{equ:9} follows from the induction hypothesis and Eq.~\ref{equ:11} follows from the distributive law due to the decomposability of $\mathcal{S}$. 
		\item 	$R$ is a sum node with weights $w_1, \ldots, w_{l}\geq 0$. We have
\begin{align}
f_{\mathcal{S}'}(\mathbf{x}) & = \sum_{i=1}^l w'_i f_{\mathcal{S}'_i}(\mathbf{x}) \\
& = \sum_{i=1}^l \frac{w_i val(R_i)}{\sum_{j=1}^l w_j val(R_j)}\Pr_{\mathcal{S}_i}(\mathbf{x}) \label{equ:13} \\
& = \sum_{i=1}^l \frac{w_i val(R_i)}{\sum_{j=1}^l w_j val(R_j)}\frac{f_{\mathcal{S}_i}(\mathbf{x})}{\sum_{\mathbf{x}}f_{\mathcal{S}_i}(\mathbf{x})} \label{equ:14}\\
& = \sum_{i=1}^l \frac{w_i val(R_i)}{\sum_{j=1}^l w_j val(R_j)}\frac{f_{\mathcal{S}_i}(\mathbf{x})}{val(R_i)} \label{equ:15}\\
& = \frac{\sum_{i=1}^l w_i f_{\mathcal{S}_i}(\mathbf{x})}{\sum_{j=1}^l w_j val(R_j)} = \frac{f_{\mathcal{S}}(\mathbf{x})}{\sum_{\mathbf{x}}f_{\mathcal{S}}(\mathbf{x})} \label{equ:16}\\
& = \Pr_{\mathcal{S}}(\mathbf{x})
\end{align}
where Eqn.~\ref{equ:13} follows from the induction hypothesis, Eq.~\ref{equ:15} and~\ref{equ:16} follow from the fact that $val(v) = \sum_{\mathbf{x}|_{\text{scope}(v)}}f_v(\mathbf{x}|_{\text{scope}(v)}), \forall v\in\mathcal{S}$. 
	\end{itemize}
\end{itemize}
This completes the proof since $\Pr_{\mathcal{S}'}(\cdot) = f_{\mathcal{S}'}(\cdot) = \Pr_{\mathcal{S}}(\cdot)$.
\end{proof}
Given a complete and decomposable SPN $\mathcal{S}$, we now construct and show that the last condition in Def.~\ref{def:normal} can be satisfied in time and space $O(\mathcal{|\mathcal{S}|})$.
\begin{lemma}
\label{lemma:dist}
Given a complete and decomposable SPN $\mathcal{S}$, there exists an SPN $\mathcal{S}'$ satisfying condition 3 in Def.~\ref{def:normal} such that $\Pr_{\mathcal{S}'}(\cdot) = \Pr_{\mathcal{S}}(\cdot)$ and $|\mathcal{S}'| = O(|\mathcal{S}|)$.
\begin{proof}
We give a proof by construction. First, if $\mathcal{S}$ is not weight normalized, apply Alg.~\ref{alg:normalization} to normalize the weights (i.e., the weights of the edges emanating from each sum node sum to 1).

Now check each sum node $v$ in $\mathcal{S}$ in a bottom-up order. If $|\text{scope}(v)| = 1$, by Corollary~\ref{coro:equal} we know the network polynomial $f_v$ is a probability distribution over its scope, say, $\{X\}$. Reduce $v$ into a terminal node which is a distribution over $X$ induced by its network polynomial and disconnect $v$ from all its children. The last step is to remove all the unreachable nodes from $\mathcal{S}$ to obtain $\mathcal{S}'$. Note that in this step we will only decrease the size of $\mathcal{S}$, hence $|\mathcal{S}'| = O(|\mathcal{S}|)$. 
\end{proof}
\end{lemma}
\begin{proof}[Proof of Thm.~\ref{thm:normal}]
The combination of Lemma~\ref{lemma:decomposability},~\ref{lemma:normalization} and~\ref{lemma:dist} completes the proof of Thm.~\ref{thm:normal}. 
\end{proof}
An example of a normal SPN constructed from the SPN in Fig.~\ref{fig:spn-example} is depicted in Fig.~\ref{fig:spn-normal}.
\begin{figure}[htb]
\centering	
	\includegraphics[width=\linewidth]{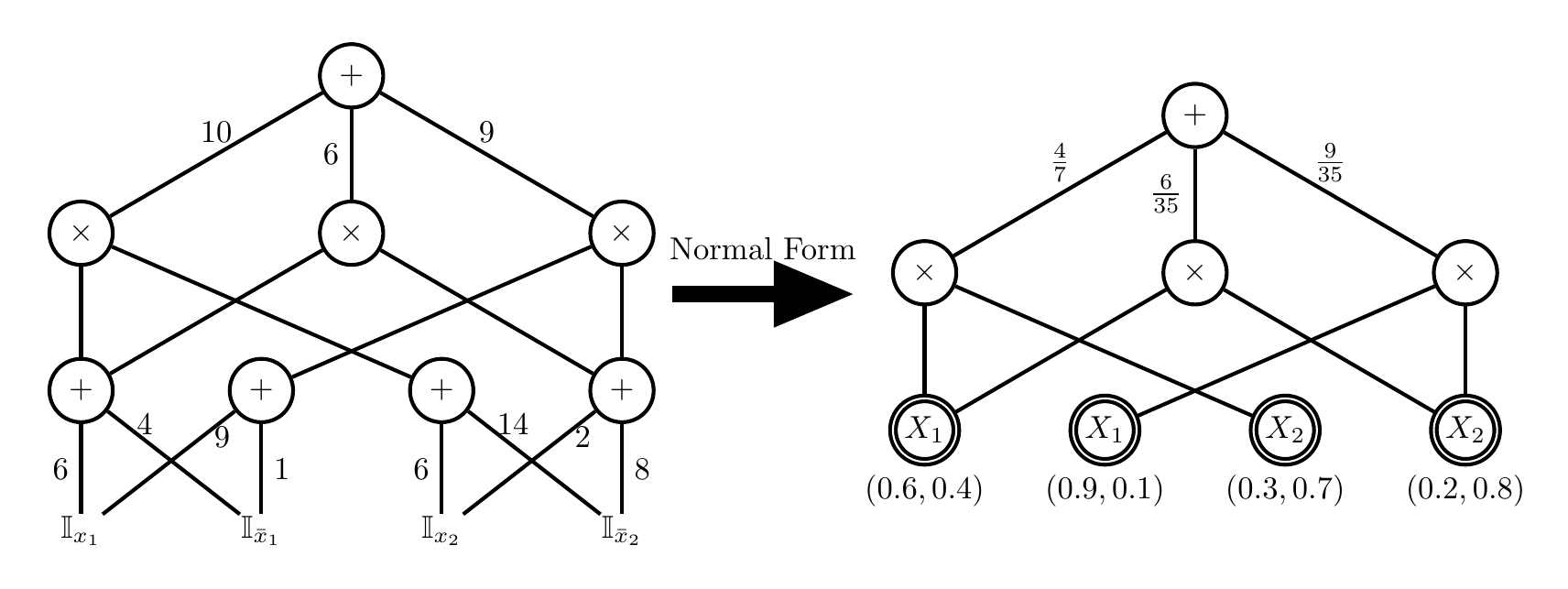}
\caption{Transform an SPN into a normal form. Terminal nodes which are probability distributions over a single variable are represented by a double-circle.}
\label{fig:spn-normal}
\end{figure}

\subsection{SPN to BN}
In order to construct a BN from an SPN, we require the SPN to be in a normal form, otherwise we can first transform it into a normal form using Alg.~\ref{alg:transformation} and~\ref{alg:normalization}.

Let $\mathcal{S}$ be a normal SPN over $\mathbf{X}_{1:N}$. Before showing how to construct a corresponding BN, we first give some intuitions. One useful view is to associate each sum node in an SPN with a hidden variable. For example, consider a sum node $v\in\mathcal{S}$ with out-degree $l$. Since $\mathcal{S}$ is normal, we have $\sum_{i=1}^l w_i = 1$ and $w_i\geq 0, \forall i\in 1:l$. This naturally suggests that we can associate a hidden discrete random variable $H_v$ with multinomial distribution $\Pr_v(H_v = i) = w_i, i\in 1:l$ for each sum node $v\in\mathcal{S}$. Therefore, $\mathcal{S}$ can be thought as defining a joint probability distribution over $\mathbf{X}_{1:N}$ and $\mathbf{H} = \{H_v~|~v\in\mathcal{S}, v\text{ is a sum node}\}$ where $\mathbf{X}_{1:N}$ are the observable variables and $\mathbf{H}$ are the hidden variables. When doing inference with an SPN, we implicitly sum out all the hidden variables $\mathbf{H}$ and compute $\Pr_{\mathcal{S}}(\mathbf{x}) = \sum_{\mathbf{h}}\Pr_{\mathcal{S}}(\mathbf{x}, \mathbf{h})$. Associating each sum node in an SPN with a hidden variable not only gives us a conceptual understanding of the probability distribution defined by an SPN, but also helps to elucidate one of the key properties implied by the structure of an SPN as summarized below:
\begin{proposition}
Given a normal SPN $\mathcal{S}$, let $p$ be a product node in $\mathcal{S}$ with $l$ children. Let $v_1, \ldots, v_k$ be sum nodes which lie on a path from the root of $\mathcal{S}$ to $p$. Then 
\begin{align}
& \Pr_{\mathcal{S}}(\mathbf{x}|_{\text{scope}(p)}~\Big|~H_{v_1} = v_1^*, \ldots, H_{v_k} = v_k^*) = \nonumber\\
& \prod_{i=1}^l\Pr_{\mathcal{S}}(\mathbf{x}|_{\text{scope}(p_i)}~\Big|~H_{v_1} = v_1^*, \ldots, H_{v_k} = v_k^*)
\label{prop:csi}
\end{align}
where $H_v = v^*$ means the sum node $v$ selects its $v^*$th branch and $\mathbf{x}|_A$ denotes restricting $\mathbf{x}$ by set $A$, $p_i$ is the $i$th child of product node $p$.
\end{proposition}
\begin{proof}
Consider the sub-SPN $\mathcal{S}_p$ rooted at $p$. $\mathcal{S}_p$ can be obtained by restricting $H_{v_1} = v_1^*,\ldots, H_{v_k} = v_k^*$, i.e., going from the root of $\mathcal{S}$ along the path $H_{v_1} = v_1^*,\ldots, H_{v_k} = v_k^*$. Since $p$ is a decomposable product node, $\mathcal{S}_p$ admits the above factorization by the definition of a product node and Corollary~\ref{coro:equal}. 
\end{proof}
Note that there may exist multiple paths from the root to $p$ in $\mathcal{S}$. Each such path admits the factorization stated in Eq.~\ref{prop:csi}. Eq.~\ref{prop:csi} explains two key insights implied by the structure of an SPN that will allow us to construct an equivalent BN with ADDs. First, CSI is efficiently encoded by the structure of an SPN using Proposition~\ref{prop:csi}. Second, the DAG structure of an SPN allows multiple assignments of hidden variables to share the same factorization, which effectively avoids the replication problem presents in decision trees.

Based on the observations above and with the help of the normal form for SPNs, we now proceed to prove the first main result in this paper: Thm.~\ref{thm:spn2bn}. First, we present the algorithm to construct the structure of a BN $\mathcal{B}$ from $\mathcal{S}$ in Alg.~\ref{alg:bn-structure}.
\begin{algorithm}[htb]
\centering
\caption{Build BN Structure}
\label{alg:bn-structure}
\begin{algorithmic}[1]
\REQUIRE normal SPN $\mathcal{S}$
\ENSURE BN $\mathcal{B} = (\mathcal{B}_V, \mathcal{B}_E)$
\STATE	$R\leftarrow$ root of $\mathcal{S}$
\IF {$R$ is a terminal node over variable $X$}
	\STATE	Create an observable variable $X$
	\STATE	$\mathcal{B}_V\leftarrow\mathcal{B}_V\cup\{X\}$
\ELSE
	\FOR	 {each child $R_i$ of $R$}
		\IF {BN has not been built for $\mathcal{S}_{R_i}$}
			\STATE Recursively build BN Structure for $\mathcal{S}_{R_i}$
		\ENDIF
	\ENDFOR
	\IF {$R$ is a sum node}
		\STATE	Create a hidden variable $H_R$ associated with $R$
		\STATE	$\mathcal{B}_V \leftarrow \mathcal{B}_V \cup \{H_R\}$
		\FOR {each observable variable $X \in \mathcal{S}_R$} 	
			\STATE $\mathcal{B}_E\leftarrow\mathcal{B}_E\cup\{(H_R, X)\}$
		\ENDFOR
	\ENDIF
\ENDIF
\end{algorithmic}
\end{algorithm}
In a nutshell, Alg.~\ref{alg:bn-structure} creates an observable variable $X$ in $\mathcal{B}$ for each terminal node over $X$ in $\mathcal{S}$ (Lines 2-4). For each internal sum node $v$ in $\mathcal{S}$, Alg.~\ref{alg:bn-structure} creates a hidden variable $H_v$ associated with $v$ and builds directed edges from $H_v$ to all observable variables $X$ appearing in the sub-SPN rooted at $v$ (Lines 11-17).\emph{ The BN $\mathcal{B}$ created by Alg.~\ref{alg:bn-structure} has a directed bipartite structure with a layer of hidden variables pointing to a layer of observable variables}. A hidden variable $H$ points to an observable variable $X$ in $\mathcal{B}$ iff $X$ appears in the sub-SPN rooted at $H$ in $\mathcal{S}$.
\begin{algorithm}[htb]
\centering
\caption{Build CPD using ADD, observable variable}
\label{alg:bn-ADD-observable}
\begin{algorithmic}[1]
\REQUIRE normal SPN $\mathcal{S}$, variable $X$
\ENSURE ADD $\mathcal{A}_X$
\IF {ADD has already been created for $\mathcal{S}$ and $X$}
	\STATE $\mathcal{A}_X\leftarrow$ retrieve ADD from cache
\ELSE
	\STATE	$R\leftarrow$ root of $\mathcal{S}$
	\IF {$R$ is a terminal node}
		\STATE $\mathcal{A}_X\leftarrow$ decision stump rooted at $R$
	\ELSIF {$R$ is a sum node}
		\STATE Create a node $H_R$ into $\mathcal{A}_X$
		\FOR {each $R_i\in Ch(R)$}
			\STATE Link $\text{BuildADD}(\mathcal{S}_{R_i}, X)$ as $i$th child of $H_R$
		\ENDFOR
	\ELSIF {$R$ is a product node}
		\STATE Find child $\mathcal{S}_{R_i}$ such that $X \in \text{scope}(R_i)$
		\STATE $\mathcal{A}_X\leftarrow \text{BuildADD}(\mathcal{S}_{R_i}, X)$
	\ENDIF
	\STATE store $\mathcal{A}_X$ in cache
\ENDIF
\end{algorithmic}
\end{algorithm}
\begin{algorithm}[htb]
\centering
\caption{Build CPD using ADD, hidden variable}
\label{alg:bn-ADD-hidden}
\begin{algorithmic}[1]
\REQUIRE normal SPN $\mathcal{S}$, variable $H$
\ENSURE ADD $\mathcal{A}_H$
\STATE	Find the sum node $H$ in $\mathcal{S}$
\STATE	$\mathcal{A}_H\leftarrow$ decision stump rooted at $H$ in $\mathcal{S}$
\end{algorithmic}
\end{algorithm}

We now present Alg.~\ref{alg:bn-ADD-observable} and~\ref{alg:bn-ADD-hidden} to build ADDs for each observable variable $X$ and hidden variable $H$ in $\mathcal{B}$. For each hidden variable $H$, Alg.~\ref{alg:bn-ADD-hidden} builds $\mathcal{A}_H$ as a decision stump\footnote{A decision stump is a decision tree with one variable.} obtained by finding $H$ and its associated weights in $\mathcal{S}$. Consider ADDs built by Alg.~\ref{alg:bn-ADD-observable} for observable variables $X$s. Let $X$ be the current observable variable we are considering. Basically, Alg.~\ref{alg:bn-ADD-observable} is a recursive algorithm applied to each node in $\mathcal{S}$ whose scope intersects with $\{X\}$. There are three cases. If current node is a terminal node, then it must be a probability distribution over $X$. In this case we simply return the decision stump at the current node. If the current node is a sum node, then due to the completeness of $\mathcal{S}$, we know that all the children of $R$ share the same scope with $R$. We first create a node $H_R$ corresponding to the hidden variable associated with $R$ into $\mathcal{A}_X$ (Line 8) and recursively apply Alg.~\ref{alg:bn-ADD-observable} to all the children of $R$ and link them to $H_R$ respectively.  If the current node is a product node, then due to the decomposability of $\mathcal{S}$, we know that there will be a unique child of $R$ whose scope intersects with $\{X\}$. We recursively apply Alg.~\ref{alg:bn-ADD-observable} to this child and return the resulting ADD (Lines 12-15).

Equivalently, Alg.~\ref{alg:bn-ADD-observable} can be understood in the following way: \emph{we extract the sub-SPN induced by $\{X\}$ and contract\footnote{In graph theory, the contraction of a node $v$ in a DAG is the operation that connects each parent of $v$ to each child of $v$ and then delete $v$ from the graph.} all the product nodes in it to obtain $\mathcal{A}_X$}. Note that the \emph{contraction} of product nodes will not add more edges into $\mathcal{A}_X$ since the out-degree of each product node in the induced sub-SPN must be 1 due to the decomposability of the product node. We illustrate the application of Alg.~\ref{alg:bn-structure},~\ref{alg:bn-ADD-observable} and~\ref{alg:bn-ADD-hidden} on the normal SPN in Fig.~\ref{fig:spn-normal}, which results in the BN $\mathcal{B}$ with CPDs represented by ADDs shown in Fig.~\ref{fig:spn2bn}.
\begin{figure*}[htb]
\centering
	\includegraphics[width=\textwidth]{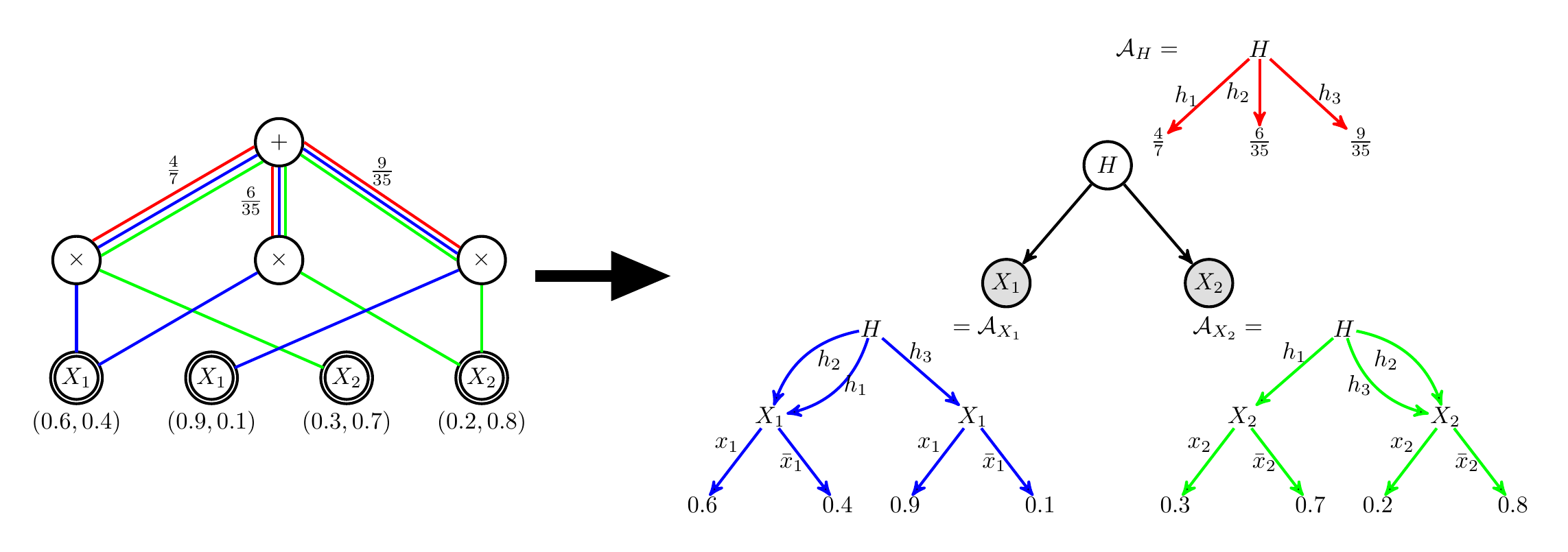}
\caption{Construct a BN with CPDs represented by ADDs from an SPN. On the left, the induced sub-SPNs used to create $\mathcal{A}_{X_1}$ and $\mathcal{A}_{X_2}$ by Alg.~\ref{alg:bn-ADD-observable} are indicated in blue and green respectively. The decision stump used to create $\mathcal{A}_H$ by Alg.~\ref{alg:bn-ADD-hidden} is indicated in red.}
\label{fig:spn2bn}
\end{figure*}

We now show that $\Pr_{\mathcal{S}}(\mathbf{x}) = \Pr_{\mathcal{B}}(\mathbf{x})\quad\forall\mathbf{x}$. 
\begin{lemma}
\label{lemma:cpdadd}
Given a normal SPN $\mathcal{S}$, the ADDs constructed by Alg.~\ref{alg:bn-ADD-observable} and~\ref{alg:bn-ADD-hidden} encode local CPDs at each node in $\mathcal{B}$.
\end{lemma}
\begin{proof}
It is easy to verify that for each hidden variable $H$ in $\mathcal{B}$, $\mathcal{A}_H$ represents a local CPD since $\mathcal{A}_H$ is a decision stump with normalized weights. 

For any observable variable $X$ in $\mathcal{B}$, let $Pa(X)$ be the set of parents of $X$. By Alg.~\ref{alg:bn-structure}, every node in $Pa(X)$ is a hidden variable. Furthermore, $\forall H$, $H\in Pa(X)$ iff there exists one terminal node over $X$ in $\mathcal{S}$ that appears in the sub-SPN rooted at $H$. Hence given any joint assignment $\mathbf{h}$ of $Pa(X)$, there will be a path in $\mathcal{A}_X$ from the root to a terminal node that is consistent with the joint assignment of the parents. Also, the leaves in $\mathcal{A}_X$ contain normalized weights corresponding to the probabilities of $X$ (see Def.~\ref{def:normal}) induced by the creation of decision stumps over $X$ in Lines 5-6 of Alg.~\ref{alg:bn-ADD-observable}.
\end{proof}

\begin{theorem}
\label{thm:equaldist}
For any normal SPN $\mathcal{S}$ over $\mathbf{X}_{1:N}$, the BN $\mathcal{B}$ constructed by Alg.~\ref{alg:bn-structure},~\ref{alg:bn-ADD-observable} and~\ref{alg:bn-ADD-hidden} encodes the same probability distribution, i.e., $\Pr_{\mathcal{S}}(\mathbf{x}) = \Pr_{\mathcal{B}}(\mathbf{x}), \forall\mathbf{x}$.
\end{theorem}
\begin{proof}
Again, we give a proof by induction on the height of $\mathcal{S}$.
\begin{itemize}
	\item 	Base case. The height of SPN $\mathcal{S}$ is 0. In this case, $\mathcal{S}$ will be a single terminal node over $X$ and $\mathcal{B}$ will be a single observable node with decision stump $\mathcal{A}_X$ constructed from the terminal node by Lines 5-6 in Alg.~\ref{alg:bn-ADD-observable}. It is clear that $\Pr_{\mathcal{S}}(x) = \Pr_{\mathcal{B}}(x),\forall x$.
	\item 	Induction step. Assume $\Pr_{\mathcal{B}}(\mathbf{x}) = \Pr_{\mathcal{S}}(\mathbf{x}),\forall\mathbf{x}$ for any $\mathcal{S}$ with height $\leq k$, where $\mathcal{B}$ is the corresponding BN constructed by Alg.~\ref{alg:bn-structure},~\ref{alg:bn-ADD-observable} and~\ref{alg:bn-ADD-hidden} from $\mathcal{S}$. Consider an SPN $\mathcal{S}$ with height $k+1$. Let $R$ be the root of $\mathcal{S}$ and $R_i, i\in 1:l$ be the children of $R$ in $\mathcal{S}$. We consider the following two cases:
	\begin{itemize}
		\item 	$R$ is a product node. Let $\text{scope}(R_t) = \mathbf{X}_t, t\in 1:l$. Claim: there is no edge between $\mathcal{S}_{i}$ and $\mathcal{S}_{j}, i\neq j$, where $\mathcal{S}_i (\mathcal{S}_j)$ is the sub-SPN rooted at $R_i (R_j)$. If there is an edge, say, from $v_j$ to $v_i$ where $v_j\in\mathcal{S}_{j}$ and $v_i\in\mathcal{S}_{i}$, then $\text{scope}(v_i)\subseteq\text{scope}(v_j)\subseteq\text{scope}(R_j)$. On the other hand, $\text{scope}(v_i)\subseteq\text{scope}(R_i)$. So we have $\varnothing\neq\text{scope}(v_i)\subseteq\text{scope}(R_i)\bigcap\text{scope}(R_j)$, which contradicts the decomposability of the product node $R$. Hence the constructed BN $\mathcal{B}$ will be a forest of $l$ disconnected components, and each component $\mathcal{B}_t$ will correspond to the sub-SPN $\mathcal{S}_t$ rooted at $R_t, \forall t\in 1:l$, with height $\leq k$. By the induction hypothesis we have $\Pr_{\mathcal{B}_t}(\mathbf{x}_t)  = \Pr_{\mathcal{S}_t}(\mathbf{x}_t), \forall t \in 1:l$. Consider the whole BN $\mathcal{B}$, we have:
\begin{equation}
\Pr_\mathcal{B}(\mathbf{x}) = \prod_t\Pr_{\mathcal{B}_t}(\mathbf{x}_t)  = \prod_t\Pr_{\mathcal{S}_t}(\mathbf{x}_t) = \Pr_\mathcal{S}(\mathbf{x})
\end{equation}
where the first equation is due to the $d$-separation rule in BNs by noting that each component $\mathcal{B}_t$ is disconnected from all other components. The second equation follows from the induction hypothesis. The last equation follows from the definition of a product node.
		\item 	$R$ is a sum node. In this case, due to the completeness of $\mathcal{S}$, all the children of $R$ share the same scope as $R$. By the construction process presented in Alg.~\ref{alg:bn-structure},~\ref{alg:bn-ADD-observable} and~\ref{alg:bn-ADD-hidden}, there is a hidden variable $H$ corresponding to $R$ that takes $l$ different values in $\mathcal{B}$. Let $w_{1:l}$ be the weights of the edges emanating from $R$ in $\mathcal{S}$. For the $t$th branch of $R$, we use $\mathbf{H}_t$ to denote the set of hidden variables in $\mathcal{B}$ that also appear in $\mathcal{B}_t$, and let $\mathbf{H}_{-t} = \mathbf{H}\backslash\mathbf{H}_t$, where $\mathbf{H}$ is the set of all hidden variables in $\mathcal{B}$ except $H$. First, we show the following identity:
\begin{align}
&\Pr_\mathcal{B}(\mathbf{x} | H = h_t) = \sum_{\mathbf{h}_t}\sum_{\mathbf{h}_{-t}}\Pr_\mathcal{B}(\mathbf{x}, \mathbf{h}_t, \mathbf{h}_{-t} | H = h_t)\\
& = \sum_{\mathbf{h}_t}\sum_{\mathbf{h}_{-t}}\Pr_\mathcal{B}(\mathbf{x}, \mathbf{h}_t | H = h_t, \mathbf{h}_{-t})\Pr_\mathcal{B}( \mathbf{h}_{-t} | H = h_t)\\
& = \sum_{\mathbf{h}_t}\sum_{\mathbf{h}_{-t}}\Pr_\mathcal{B}(\mathbf{x}, \mathbf{h}_t | H = h_t)\Pr_\mathcal{B}( \mathbf{h}_{-t} | H = h_t)\label{equ:csii}\\
& = \sum_{\mathbf{h}_t}\Pr_\mathcal{B}(\mathbf{x}, \mathbf{h}_t | H = h_t)\sum_{\mathbf{h}_{-t}}\Pr_\mathcal{B}( \mathbf{h}_{-t} | H = h_t)\label{equ:ve}\\
& = \sum_{\mathbf{h}_t}\Pr_\mathcal{B}(\mathbf{x}, \mathbf{h}_t | H = h_t) \\
& = \sum_{\mathbf{h}_t}\Pr_{\mathcal{B}_t}(\mathbf{x}, \mathbf{h}_t) = \Pr_{\mathcal{B}_t}(\mathbf{x})\label{equ:ttt}
\end{align}		
Using this identity, we have
\begin{align}
\Pr_\mathcal{B}(\mathbf{x}) &= \sum_{t=1}^l\Pr_{\mathcal{B}}(h_t)\Pr_{\mathcal{B}}(\mathbf{x} | H = h_t) \\
& = \sum_{t=1}^l w_t\Pr_{\mathcal{B}_t}(\mathbf{x}) \label{equ:id}\\
& = \sum_{t=1}^l w_t\Pr_{\mathcal{S}_t}(\mathbf{x}) \label{equ:ind}\\
& = \Pr_{\mathcal{S}}(\mathbf{x})
\end{align}
Eq.~\ref{equ:csii} follows from the fact that $\mathbf{X}$ and $\mathbf{H}_t$ are independent of $\mathbf{H}_{-t}$ given $H = h_t$, i.e., we take advantage of the CSI described by ADDs of $\mathbf{X}$. Eq.~\ref{equ:ve} follows from the fact that $\mathbf{H}_{-t}$ appears only in the second term. Combined with the fact that $H = h_t$ is given as evidence in $\mathcal{B}$, this gives us the induced subgraph $\mathcal{B}_t$ referred to in Eq.~\ref{equ:ttt}. Eq.~\ref{equ:id} follows from Eq.~\ref{equ:ttt} and Eq.~\ref{equ:ind} follows from the induction hypothesis.
	\end{itemize}
\end{itemize}
Combing the base case and the induction step completes the proof for Thm.~\ref{thm:equaldist}.
\end{proof}
We now bound the size of $\mathcal{B}$:
\begin{theorem}
\label{thm:size}
$|\mathcal{B}| = O(N|\mathcal{S}|)$, where BN $\mathcal{B}$ is constructed by Alg.~\ref{alg:bn-structure},~\ref{alg:bn-ADD-observable} and~\ref{alg:bn-ADD-hidden} from normal SPN $\mathcal{S}$ over $\mathbf{X}_{1:N}$.
\end{theorem}
\begin{proof}
For each observable variable $X$ in $\mathcal{B}$, $\mathcal{A}_X$ is constructed by first extracting from $\mathcal{S}$ the induced sub-SPN $\mathcal{S}_X$ that contains all nodes whose scope includes $X$ and then contracting all the product nodes in $\mathcal{S}_X$ to obtain $\mathcal{A}_X$. By the decomposability of product nodes, each product node in $\mathcal{S}_X$ has out-degree 1 otherwise the original SPN $\mathcal{S}$ violates the decomposability property. Since contracting product nodes does not increase the number of edges in $\mathcal{S}_X$, we have $|\mathcal{A}_X|\leq |\mathcal{S}_X|\leq|\mathcal{S}|$.

For each hidden variable $H$ in $\mathcal{B}$, $\mathcal{A}_H$ is a decision stump constructed from the internal sum node corresponding to $H$ in $\mathcal{S}$. Hence, we have $\sum_{H}\mathcal{A}_H\leq |\mathcal{S}|$.

Now consider the size of the graph $\mathcal{B}$. Note that only terminal nodes and sum nodes will have corresponding variables in $\mathcal{B}$. It is clear that the number of nodes in $\mathcal{B}$ is bounded by the number of nodes in $\mathcal{S}$. Furthermore, a hidden variable $H$ points to an observable variable $X$ in $\mathcal{B}$ iff $X$ appears in the sub-SPN rooted at $H$ in $\mathcal{S}$, i.e., there is a path from the sum node corresponding to $H$ to one of the terminal nodes in $X$. For a sum node $H$ (which corresponds to a hidden variable $H\in\mathcal{B}$) with scope size $s$, each edge emanated from $H$ in $\mathcal{S}$ will correspond to directed edges in $\mathcal{B}$ at most $s$ times, since there are exactly $s$ observable variables which are children of $H$ in $\mathcal{B}$. It is clear that $s\leq N$, so each edge emanated from a sum node in $\mathcal{S}$ will be counted at most $N$ times in $\mathcal{B}$. Edges from product nodes will not occur in the graph of $\mathcal{B}$, instead, they have been counted in the ADD representations of the local CPDs in $\mathcal{B}$. So again, the size of the graph $\mathcal{B}$ is bounded by $\sum_{H}\text{scope}(H)\times\text{deg}(H) \leq \sum_{H}N\text{deg}(H)\leq 2N|\mathcal{S}|$. 

There are $N$ observable variables in $\mathcal{B}$. So the total size of $\mathcal{B}$, including the size of the graph and the size of all the ADDs, is bounded by $N|\mathcal{S}| + |\mathcal{S}| + 2N|\mathcal{S}| = O(N|\mathcal{S}|)$.
\end{proof}
We give the time complexity of Alg.~\ref{alg:bn-structure},~\ref{alg:bn-ADD-observable} and~\ref{alg:bn-ADD-hidden}.
\begin{theorem}
\label{thm:time}
For any normal SPN $\mathcal{S}$ over $\mathbf{X}_{1:N}$, Alg.~\ref{alg:bn-structure},~\ref{alg:bn-ADD-observable} and~\ref{alg:bn-ADD-hidden} construct an equivalent BN in time $O(N|\mathcal{S}|)$.
\end{theorem}
\begin{proof}
First consider Alg.~\ref{alg:bn-structure}. Alg.~\ref{alg:bn-structure} recursively visits each node and its children in $\mathcal{S}$ if they have not been visited (Lines 6-10). For each node $v$ in $\mathcal{S}$, Lines 7-9 cost at most $2\cdot\text{out-degree}(v)$. If $v$ is a sum node, then Lines 11-17 create a hidden variable and then connect the hidden variable to all observable variables that appear in the sub-SPN rooted at $v$, which is clearly bounded by the number of all observable variables, $N$. So the total cost of Alg.~\ref{alg:bn-structure} is bounded by $\sum_{v} 2\cdot\text{out-degree}(v) + \sum_{v\text{ is a sum node}}N \leq 2\mathfrak{V}(\mathcal{S}) + 2\mathfrak{E}(\mathcal{S}) + N\mathfrak{V}(\mathcal{S})\leq 2|\mathcal{S}| + N|\mathcal{S}| = O(N|\mathcal{S}|)$. Note that we assume that inserting an element into a set can be done in $O(1)$ by using hashing.

The analysis for Alg.~\ref{alg:bn-ADD-observable} and~\ref{alg:bn-ADD-hidden} follows from the same analysis as in the proof for Thm.~\ref{thm:size}. The time complexity for Alg.~\ref{alg:bn-ADD-observable} and Alg.~\ref{alg:bn-ADD-hidden} is then bounded by $N|\mathcal{S}| + |\mathcal{S}| = O(N|\mathcal{S}|)$.
\end{proof}
\begin{proof}[Proof of Thm.~\ref{thm:spn2bn}]
The combination of Thm.~\ref{thm:equaldist},~\ref{thm:size} and~\ref{thm:time} proves Thm.~\ref{thm:spn2bn}.
\end{proof}
\begin{proof}[Proof of Corollary.~\ref{coro:spn2bn}]
Given a complete and consistent SPN $\mathcal{S}$, we can first transform it into a normal SPN $\mathcal{S}'$ with $|\mathcal{S}'| = O(|\mathcal{S}|^2)$ by Thm.~\ref{thm:normal} if it is not normal. After this the analysis follows from Thm.~\ref{thm:spn2bn}.
\end{proof}

\subsection{BN to SPN}
It is known that a BN with CPDs represented by tables can be converted into an SPN by first converting the BN into a junction tree and then translating the junction tree into an SPN. The size of the generated SPN, however, will be exponential in the tree-width of the original BN since the tabular representation of CPDs is ignorant of CSI. As a result, the generated SPN loses its power to compactly represent some BNs with high tree-width, yet, with CSI in its local CPDs. 

Alternatively, one can also compile a BN with ADDs into an AC~\cite{chavira2007compiling} and then convert an AC into an SPN~\cite{rooshenas2014learning}. However, in \citet{chavira2007compiling}'s compilation approach, the variables appearing along a path from the root to a leaf in each ADD must be consistent with a pre-defined global variable ordering. The global variable ordering, may, to some extent restrict the compactness of ADDs as the most compact representation for different ADDs normally have different topological orderings. Interested readers are referred to \cite{chavira2007compiling} for more details on this topic.

In this section, we focus on BNs with ADDs that are constructed using Alg.~\ref{alg:bn-ADD-observable} and~\ref{alg:bn-ADD-hidden} from normal SPNs. We show that when applying VE to those BNs with ADDs we can recover the original normal SPNs. \emph{The key insight is that the structure of the original normal SPN naturally defines a global variable ordering that is consistent with the topological ordering of every ADD constructed}. More specifically, since all the ADDs constructed using Alg.~\ref{alg:bn-ADD-observable} are induced sub-SPNs with contraction of product nodes from the original SPN $\mathcal{S}$, the topological ordering of all the nodes in $\mathcal{S}$ can be used as the pre-defined variable ordering for all the ADDs. 
\begin{algorithm}[htb!]
\centering
\caption{Multiplication of two symbolic ADDs, $\otimes$}
\label{alg:multiplication}
\begin{algorithmic}[1]
\REQUIRE	Symbolic ADD $\mathcal{A}_{X_1}$, $\mathcal{A}_{X_2}$
\ENSURE	Symbolic ADD $\mathcal{A}_{X_1, X_2} = \mathcal{A}_{X_1}\otimes\mathcal{A}_{X_2}$
\STATE	$R_1\leftarrow$ root of $\mathcal{A}_{X_1}$, $R_2\leftarrow$ root of $\mathcal{A}_{X_2}$
\IF {$R_1$ and $R_2$ are both variable nodes}
    \IF {$R_1= R_2$}
    	\STATE	Create a node $R=R_1$ into $\mathcal{A}_{X_1, X_2}$
    	\FOR {each $r\in dom(R)$}
    		\STATE	$\mathcal{A}_{X_1}^r\leftarrow Ch(R_1)|_r$
    		\STATE	$\mathcal{A}_{X_2}^r\leftarrow Ch(R_2)|_r$
    		\STATE	$\mathcal{A}_{X_1,X_2}^r\leftarrow\mathcal{A}_{X_1}^r\otimes\mathcal{A}_{X_2}^r$
    		\STATE	Link $\mathcal{A}_{X_1, X_2}^r$ to the $r$th child of $R$ in $\mathcal{A}_{X_1, X_2}$  
    	\ENDFOR
    \ELSE
    	\STATE 	$\mathcal{A}_{X_1, X_2}\leftarrow$ create a symbolic node $\otimes$
    	\STATE	Link $\mathcal{A}_{X_1}$ and $\mathcal{A}_{X_2}$ as two children of $\otimes$
    \ENDIF
\ELSIF {$R_1$ is a variable node and $R_2$ is $\otimes$}
    \IF {$R_1$ appears as a child of $R_2$}
		\STATE	$\mathcal{A}_{X_1, X_2}\leftarrow\mathcal{A}_{X_2}$
		\STATE	$\mathcal{A}_{X_1,X_2}^{R_1}\leftarrow\mathcal{A}_{X_1}\otimes\mathcal{A}_{X_2}^{R_1}$
	\ELSE
    		\STATE  Link $\mathcal{A}_{X_1}$ as a new child of $R_2$
		\STATE	$\mathcal{A}_{X_1, X_2}\leftarrow\mathcal{A}_{X_2}$
	\ENDIF
\ELSIF {$R_1$ is $\otimes$ and $R_2$ is a variable node}
    \IF {$R_2$ appears as a child of $R_1$}
		\STATE	$\mathcal{A}_{X_1, X_2}\leftarrow\mathcal{A}_{X_1}$
		\STATE	$\mathcal{A}_{X_1,X_2}^{R_2}\leftarrow\mathcal{A}_{X_2}\otimes\mathcal{A}_{X_1}^{R_2}$
	\ELSE
    		\STATE  Link $\mathcal{A}_{X_2}$ as a new child of $R_1$
		\STATE	$\mathcal{A}_{X_1, X_2}\leftarrow\mathcal{A}_{X_1}$
	\ENDIF
\ELSE 
    \STATE  $\mathcal{A}_{X_1, X_2}\leftarrow$ create a symbolic node $\otimes$
    \STATE  Link $\mathcal{A}_{X_1}$ and $\mathcal{A}_{X_2}$ as two children of $\otimes$
\ENDIF
\STATE	Merge connected product nodes in $\mathcal{A}_{X_1, X_2}$
\end{algorithmic}
\end{algorithm}

\begin{algorithm}[htb!]
\centering
\caption{Summing-out a hidden variable $H$ from $\mathcal{A}$ using $\mathcal{A}_H$, $\oplus$}
\label{alg:sumout}
\begin{algorithmic}[1]
\REQUIRE		Symbolic ADDs $\mathcal{A}$ and $\mathcal{A}_H$
\ENSURE		Symbolic ADD with $H$ summed out
\IF {$H$ appears in $\mathcal{A}$}
	\STATE	Label each edge emanating from $H$ with weights obtained from $\mathcal{A}_H$ 
	\STATE	Replace $H$ by a symbolic $\oplus$ node
\ENDIF
\end{algorithmic}
\end{algorithm}

In order to apply VE to a BN with ADDs, we need to show how to apply two common operations used in VE, i.e., multiplication of two factors and summing-out a hidden variable, on ADDs. For our purpose, we use a \emph{symbolic ADD} as an intermediate representation during the inference process of VE by allowing symbolic operations, such as $+, -, \times, /$ to appear as internal nodes in ADDs. In this sense, an ADD can be viewed as a special type of symbolic ADD where all the internal nodes are variables. The same trick was applied by \cite{chavira2007compiling} in their compilation approach. For example, given symbolic ADDs $\mathcal{A}_{X_1}$ over $X_1$ and $\mathcal{A}_{X_2}$ over $X_2$, Alg.~\ref{alg:multiplication} returns a symbolic ADD $\mathcal{A}_{X_1,X_2}$ over $X_1, X_2$ such that $\mathcal{A}_{X_1, X_2}(x_1, x_2) \triangleq\left(\mathcal{A}_{X_1}\otimes\mathcal{A}_{X_2}\right)(x_1,x_2) = \mathcal{A}_{X_1}(x_1)\times\mathcal{A}_{X_2}(x_2)$. To simplify the presentation, we choose the inverse topological ordering of the hidden variables in the original SPN $\mathcal{S}$ as the elimination order used in VE. This helps to avoid the situations where a multiplication is applied to a sum node in symbolic ADDs. Other elimination orders could be used, but a more detailed discussion of sum nodes is needed.

Given two symbolic ADDs $\mathcal{A}_{X_1}$ and $\mathcal{A}_{X_2}$, Alg.~\ref{alg:multiplication} recursively visits nodes in $\mathcal{A}_{X_1}$ and $\mathcal{A}_{X_2}$ simultaneously. In general, there are 3 cases: 1) the roots of $\mathcal{A}_{X_1}$ and $\mathcal{A}_{X_2}$ are both variable nodes (Lines 2-14); 2) one of the two roots is a variable node and the other is a product node (Lines 15-30); 3) both roots are product nodes or at least one of them is a sum node (Lines 31-34). We discuss these 3 cases. 

If both roots of $\mathcal{A}_{X_1}$ and $\mathcal{A}_{X_2}$ are variable nodes, there are two subcases to be considered. First, if they are nodes labeled with the same variable (Lines 3-10), then the computation related to the common variable is shared and the multiplication is recursively applied to all the children, otherwise we simply create a symbolic product node $\otimes$ and link $\mathcal{A}_{X_1}$ and $\mathcal{A}_{X_2}$ as its two children (Lines 11-14). Once we find $R_1\in\mathcal{A}_{X_1}$ and $R_2\in\mathcal{A}_{X_2}$ such that $R_1\neq R_2$, there will be no common node that is shared by the sub-ADDs rooted at $R_1$ and $R_2$. To see this, note that Alg.~\ref{alg:multiplication} recursively calls itself as long as the roots of $\mathcal{A}_{X_1}$ and $\mathcal{A}_{X_2}$ are labeled with the same variable. Let $R$ be the last variable shared by the roots of $\mathcal{A}_{X_1}$ and $\mathcal{A}_{X_2}$ in Alg.~\ref{alg:multiplication}. Then $R_1$ and $R_2$ must be the children of $R$ in the original SPN $\mathcal{S}$. Since $R_1$ does not appear in $\mathcal{A}_{X_2}$, then $X_2\not\in\text{scope}(R_1)$, otherwise $R_1$ will occur in $\mathcal{A}_{X_2}$ and $R_1$ will be a new shared variable below $R$, which is a contradiction to the fact that $R$ is the last shared variable. Since $R_1$ is the root of the sub-ADD of $\mathcal{A}_{X_1}$ rooted at $R$, hence no variable whose scope contains $X_2$ will occur as a descendant of $R_1$, otherwise the scope of $R_1$ will also contain $X_2$, which is again a contradiction. On the other hand, each node appearing in $\mathcal{A}_{X_2}$ corresponds to a variable whose scope intersects with $\{X_2\}$ in the original SPN, hence no node in $\mathcal{A}_{X_2}$ will appear in $\mathcal{A}_{X_1}$. The same analysis also applies to $R_2$. Hence no node will be shared between $\mathcal{A}_{X_1}$ and $\mathcal{A}_{X_2}$.

If one of the two roots, say, $R_1$, is a variable node and the other root, say, $R_2$, is a product node, then we consider two subcases.  If $R_1$ appears as a child of $R_2$ then we recursively multiply $R_1$ with the child of $R_2$ that is labeled with the same variable as $R_1$ (Lines 16-18).  If $R_1$ does not appear as a child of $R_2$, then we link the ADD rooted at $R_1$ to be a new child of the product node $R_2$ (Lines 19-22). Again, let $R$ be the last shared node between $\mathcal{A}_{X_1}$ and $\mathcal{A}_{X_2}$ during the multiplication process. Then both $R_1$ and $R_2$ are children of $R$, which corresponds to a sum node in the original SPN $\mathcal{S}$. Furthermore, both $R_1$ and $R_2$ lie in the same branch of $R$ in $\mathcal{S}$. In this case, since $\text{scope}(R_1)\subseteq\text{scope}(R)$, $\text{scope}(R_1)$ must be a strict subset of $\text{scope}(R)$ otherwise we would have $\text{scope}(R_1) = \text{scope}(R)$ and $R_1$ will also appear in $\mathcal{A}_{X_2}$, which contradicts the fact that $R$ is the last shared node between $\mathcal{A}_{X_1}$ and $\mathcal{A}_{X_2}$. Hence here we only need to discuss the two cases where either their scope disjoint (Line 16-18) or the scope of one root is a strict subset of another (Line 19-22).

If the two roots are both product nodes or at least one of them is a sum node, then we simply create a new product node and link $\mathcal{A}_{X_1}$ and $\mathcal{A}_{X_2}$ to be children of the product node. The above analysis also applies here since sum nodes in symbolic ADD are created by summing out processed variable nodes and we eliminate all the hidden variables using the inverse topological ordering.

The last step in Alg.~\ref{alg:multiplication} (Line 35) simplifies the symbolic ADD by merging all the connected product nodes without changing the function it encodes. This can be done in the following way: suppose $\otimes_1$ and $\otimes_2$ are two connected product nodes in symbolic ADD $\mathcal{A}$ where $\otimes_1$ is the parent of $\otimes_2$, then we can remove the link between $\otimes_1$ and $\otimes_2$ and connect $\otimes_1$ to every child of $\otimes_2$. It is easy to verify that such an operation will remove links between connected product nodes while keeping the encoded function unchanged. 

To sum-out one hidden variable $H$, Alg.~\ref{alg:sumout} simply replaces $H$ in $\mathcal{A}$ by a symbolic sum node $\oplus$ and labels each edge of $\oplus$ with weights obtained from $\mathcal{A}_H$.

We now present the Variable Elimination (VE) algorithm in Alg.~\ref{alg:veadd} used to recover the original SPN $\mathcal{S}$, taking Alg.~\ref{alg:multiplication} and Alg.~\ref{alg:sumout} as two operations $\otimes$ and $\oplus$ respectively.
\begin{algorithm}[htb]
\centering
\caption{Variable Elimination for BN with ADDs}
\label{alg:veadd}
\begin{algorithmic}[1]
\REQUIRE	BN $\mathcal{B}$ with ADDs for all observable variables and hidden variables
\ENSURE	Original SPN $\mathcal{S}$
\STATE	$\pi\leftarrow$ the inverse topological ordering of all the hidden variables present in the ADDs
\STATE	$\Phi\leftarrow\{\mathcal{A}_X~|~X\text{ is an observable variable}\}$
\FOR {each hidden variable $H$ in $\pi$}
	\STATE	$P\leftarrow\{\mathcal{A}_X~|~ H\text{ appears in }\mathcal{A}_X\}$
	\STATE	$\Phi\leftarrow\Phi\backslash P$ $\cup$ $\{\oplus_{H}\otimes_{\mathcal{A}\in P}\mathcal{A}\}$
\ENDFOR 
\STATE \textbf{return} 	$\Phi$
\end{algorithmic}
\end{algorithm}

In each iteration of Alg.~\ref{alg:veadd}, we select one hidden variable $H$ in ordering $\pi$, multiply all the ADDs $\mathcal{A}_X$ in which $H$ appears using Alg.~\ref{alg:multiplication} and then sum-out $H$ using Alg.~\ref{alg:sumout}. The algorithm keeps going until all the hidden variables have been summed out and there is only one symbolic ADD left in $\Phi$. The final symbolic ADD gives us the SPN $\mathcal{S}$ which can be used to build BN $\mathcal{B}$. Note that the SPN returned by Alg.~\ref{alg:veadd} may not be literally equal to the original SPN since during the multiplication of two symbolic ADDs we effectively remove redundant nodes by merging connected product nodes. Hence, the SPN returned by Alg.~\ref{alg:veadd} could have a smaller size while representing the same probability distribution. An example is given in Fig.~\ref{fig:bn2spn} to illustrate the recovery process. The BN in Fig.~\ref{fig:bn2spn} is the one constructed in Fig.~\ref{fig:spn2bn}.
\begin{figure*}[htb]
\centering
	\includegraphics[width=\textwidth]{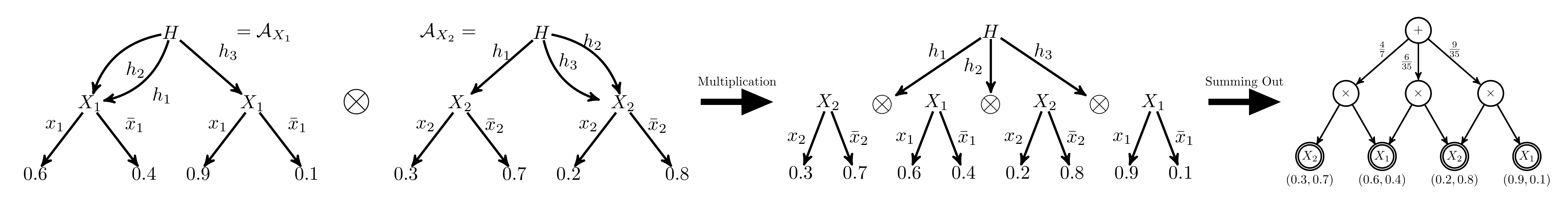}
\caption{Multiply $\mathcal{A}_{X_1}$ and $\mathcal{A}_{X_2}$ that contain $H$ using Alg.~\ref{alg:multiplication} and then sum out $H$ by applying Alg.~\ref{alg:sumout}. The final SPN is isomorphic with the SPN in Fig.~\ref{fig:spn2bn}.}
\label{fig:bn2spn}
\end{figure*}

Note that Alg.~\ref{alg:multiplication} and~\ref{alg:sumout} apply only to ADDs constructed from normal SPNs by Alg.~\ref{alg:bn-ADD-observable} and~\ref{alg:bn-ADD-hidden} because such ADDs naturally inherit the topological ordering of sum nodes (hidden variables) in the original SPN $\mathcal{S}$. Otherwise we need to pre-define a global variable ordering of all the sum nodes and then arrange each ADD such that its topological ordering is consistent with the pre-defined ordering. Note also that Alg.~\ref{alg:multiplication} and~\ref{alg:sumout} should be implemented with caching of repeated operations in order to ensure that directed acyclic graphs are preserved.  Alg.~\ref{alg:veadd} suggests that an SPN can also be viewed as a \emph{history record} or \emph{caching} of the sums and products computed during inference when applied to the resulting BN with ADDs. 

We now bound the run time of Alg.~\ref{alg:veadd}.
\begin{theorem}
\label{thm:main2}
Alg.~\ref{alg:veadd} builds SPN $\mathcal{S}$ from BN $\mathcal{B}$ with ADDs in $O(N|\mathcal{S}|)$.
\end{theorem}
\begin{proof}
First, it is easy to verify that Alg.~\ref{alg:multiplication} takes at most $|\mathcal{A}_{X_1}| + |\mathcal{A}_{X_2}|$ operations to compute the multiplication of $\mathcal{A}_{X_1}$ and $\mathcal{A}_{X_2}$. More importantly, the size of the generated $\mathcal{A}_{X_1, X_2}$ is also bounded by $|\mathcal{S}|$. This is because all the common nodes and edges in $\mathcal{A}_{X_1}$ and $\mathcal{A}_{X_2}$ are shared (not duplicated) in $\mathcal{A}_{X_1, X_2}$.  Also, all the other nodes and edges which are not shared between $\mathcal{A}_{X_1}$ and $\mathcal{A}_{X_2}$ will be in two branches of a product node in $\mathcal{S}$, otherwise they will be shared by $\mathcal{A}_{X_1}$ and $\mathcal{A}_{X_2}$ as they have the same scope which contain both $X_1$ and $X_2$. This means that $\mathcal{A}_{X_1, X_2}$ can be viewed as a sub-SPN of $\mathcal{S}$ induced by the node set $\{X_1, X_2\}$ with some product nodes contracted out. So we have $|\mathcal{A}_{X_1, X_2}|\leq|\mathcal{S}|$.

Now consider the for loop (Lines 3-6) in Alg.~\ref{alg:veadd}. The loop ends once we've summed out all the hidden variables and there is only one ADD left. Note that there may be only one ADD in $\Phi$ during some intermediate steps, in which case we do not have to do any multiplication. In such steps, we only need to perform the sum out procedure without multiplying ADDs. Since there are $N$ ADDs at the beginning of the loop and after the loop we only have one ADD, then there is exactly $N-1$ multiplications during the for loop, which costs at most $(N-1)|\mathcal{S}|$ operations. Furthermore, in each iteration there is exactly one hidden variable being summed out. So the total cost for summing out all the hidden variables in Lines 3-6 is bounded by $|\mathcal{S}|$. 

Overall, the operations in Alg.~\ref{alg:veadd} are bounded by $(N-1)|\mathcal{S}| + |\mathcal{S}| = O(N|\mathcal{S}|)$.
\end{proof}
\begin{proof}[Proof of Thm.~\ref{thm:bn2spn}]
Thm.~\ref{thm:main2} and the analysis above prove Thm.~\ref{thm:bn2spn}.
\end{proof}

\section{Discussion}
Thm.~\ref{thm:spn2bn} together with Thm.~\ref{thm:bn2spn} establish a relationship between BNs and SPNs: SPNs are no more powerful than BNs with ADD representation. Informally, a model is considered to be more powerful than another if there exists a distribution that can be encoded in polynomial size in some input parameter $N$, while the other model requires exponential size in $N$ to represent the same distribution. The key is to recognize that the CSI encoded by the structure of an SPN as stated in Proposition.~\ref{prop:csi} can also be encoded explicitly with ADDs in a BN. We can also view an SPN as an inference machine that efficiently records the history of the inference process when applied to a BN. Based on this perspective, an SPN is actually storing the calculations to be performed (sums and products), which allows online inference queries to be answered quickly.  The same idea also exists in other fields, including propositional logic (d-DNNF) and knowledge compilation (AC). 

The constructed BN has a simple bipartite structure, no matter how deep the original SPN is. However, we can relate the depth of an SPN to a lower bound on the tree-width of the corresponding BN obtained by our algorithm. Without loss of generality, let's assume that product layers alternate with sum layers in the SPN we are considering. Let the height of the SPN, i.e., the longest path from the root to a terminal node, be $K$. By our assumption, there will be at least $\lfloor K/2\rfloor$ sum nodes in the longest path. Accordingly, in the BN constructed by Alg.~\ref{alg:bn-structure}, the observable variable corresponding to the terminal node in the longest path will have in-degree at least $\lfloor K/2\rfloor$. Hence, after moralizing the BN into an undirected graph, the clique-size of the moral graph is bounded below by $\lfloor K/2\rfloor + 1$. Note that for any undirected graph the clique-size minus 1 is always a lower bound of the tree-width. We then reach the conclusion that the tree-width of the constructed BN has a lower bound of $\lfloor K/2\rfloor$. In other words, \emph{the deeper the SPN, the larger the tree-width of the BN constructed by our algorithm and the more complex are the probability distributions that can be encoded}. This observation is consistent with the conclusion drawn in~\cite{delalleau2011shallow} where the authors prove that there exist families of distributions that can be represented much more efficiently with a deep SPN than with a shallow one, i.e. with substantially fewer hidden internal sum nodes. Note that we only give a proof that there exists an algorithm that can convert an SPN into a BN without any exponential blow-up. There may exist other techniques to convert an SPN into a BN with a more compact representation and also a smaller tree-width.

High tree-width is usually used to indicate a high inference complexity, but this is not always true as there may exist lots of CSI between variables, which can reduce inference complexity. CSI is precisely what enables SPNs and BNs with ADDs to compactly represent and tractably perform inference in distributions with high tree-width. In contrast, in a Restricted Boltzmann Machine, which is an undirected bipartite Markov network, CSI may not be present or not exploited, which is why practitioners have to resort to approximate algorithms, such as contrastive divergence~\cite{carreira2005contrastive}. Similarly, approximate inference is required in bipartite diagnostic BNs such as the Quick Medical Reference network~\cite{shwe1991probabilistic} since causal independence is insufficient to reduce the complexity, while CSI is not present or not exploited.

\section{Conclusion}
In this paper, we establish a precise connection between BNs and SPNs by providing a constructive algorithm to transform between these two models. To simplify the proof, we introduce the notion of normal SPN and describe the relationship between consistency and decomposability in SPNs. We analyze the impact of the depth of SPNs onto the tree-width of the corresponding BNs. Our work also provides a new direction for future research about SPNs and BNs.  Structure and parameter learning algorithms for SPNs can now be used to indirectly learn BNs with ADDs.  In the resulting BNs, correlations are not expressed by links directly between observed variables, but rather through hidden variables that are ancestors of correlated observed variables. The structure of the resulting BNs can be used to study probabilistic dependencies and causal relationships between the variables of the original SPNs.  It would also be interesting to explore the opposite direction since there is already a large literature on parameter and structure learning for BNs.  One could learn a BN from data and then exploit CSI to convert it into an SPN.

\bibliography{reference}
\bibliographystyle{icml2014}
\end{document}